\documentclass[sigconf,authorversion,nonacm]{acmart}

\usepackage{booktabs} 

\usepackage[english]{babel}
\usepackage{listings}

\usepackage{multirow, multicol}
\usepackage{amsmath,amsfonts, scalerel}

\usepackage{graphicx}
\usepackage{textcomp}
\usepackage{xcolor}
\usepackage{algorithm,algpseudocode}
\usepackage{balance}

\usepackage[shortlabels]{enumitem}
\usepackage[font=small,skip=0pt]{caption}
\usepackage{tikz,pgfplots}\pgfplotsset{compat=1.18}
\usepackage{tikz}
\usetikzlibrary{trees}
\usepackage{lipsum,adjustbox}
\usetikzlibrary{positioning}
\usetikzlibrary{mindmap,trees}
\usepackage{verbatim}

\usepackage[T1]{fontenc}
\usepackage{multicol}
\usepackage[skins]{tcolorbox}
\newtcolorbox{myframe}[2][]{%
	enhanced,colback=white,colframe=black,coltitle=black,
	sharp corners,boxrule=0.6pt,
	fonttitle=\itshape,
	attach boxed title to top left={yshift=-0.3\baselineskip-0.4pt,xshift=2mm},
	boxed title style={tile,size=minimal,left=0.5mm,right=0.5mm,
		colback=white,before upper=\strut},
	title=#2,#1
}

\usepackage{footnote}
\makesavenoteenv{tabular}
\makesavenoteenv{table}

\usepackage{boldline} 
\usepackage{color, colortbl}

\definecolor{Gray}{gray}{0.9}
\definecolor{airforceblue}{rgb}{0.36, 0.54, 0.66}
\definecolor{aliceblue}{rgb}{0.94, 0.97, 1.0}
\definecolor{alizarin}{rgb}{0.82, 0.1, 0.26}
\definecolor{amber}{rgb}{1.0, 0.75, 0.0}
\definecolor{amber(sae/ece)}{rgb}{1.0, 0.49, 0.0}
\definecolor{antiquebrass}{rgb}{0.8, 0.58, 0.46}
\definecolor{azure(web)(azuremist)}{rgb}{0.94, 1.0, 1.0}
\definecolor{bronze}{rgb}{0.8, 0.5, 0.2}
\definecolor{battleshipgrey}{rgb}{0.52, 0.52, 0.51}
\definecolor{bole}{rgb}{0.47, 0.27, 0.23}
\definecolor{bulgarianrose}{rgb}{0.28, 0.02, 0.03}
\definecolor{cadet}{rgb}{0.33, 0.41, 0.47}
\definecolor{ceil}{rgb}{0.57, 0.63, 0.81}
\definecolor{cerulean}{rgb}{0.0, 0.48, 0.65}
\definecolor{charcoal}{rgb}{0.21, 0.27, 0.31}
\definecolor{coolblack}{rgb}{0.0, 0.18, 0.39}
\definecolor{coolgrey}{rgb}{0.55, 0.57, 0.67}
\definecolor{darkcandyapplered}{rgb}{0.64, 0.0, 0.0}
\definecolor{darkbrown}{rgb}{0.4, 0.26, 0.13}
\definecolor{darkcerulean}{rgb}{0.03, 0.27, 0.49}
\definecolor{darkgray}{rgb}{0.66, 0.66, 0.66}
\definecolor{darkjunglegreen}{rgb}{0.1, 0.14, 0.13}
\definecolor{darktaupe}{rgb}{0.28, 0.24, 0.2}
\definecolor{davy\'sgrey}{rgb}{0.33, 0.33, 0.33}
\definecolor{frenchblue}{rgb}{0.0, 0.45, 0.73}
\definecolor{almond}{rgb}{0.94, 0.87, 0.8}
\definecolor{beaublue}{rgb}{0.74, 0.83, 0.9}
\definecolor{beige}{rgb}{0.96, 0.96, 0.86}
\definecolor{bisque}{rgb}{1.0, 0.89, 0.77}
\definecolor{black}{rgb}{0.0, 0.0, 0.0}
\definecolor{fluorescentorange}{rgb}{1.0, 0.75, 0.0}
\definecolor{ghostwhite}{rgb}{0.97, 0.97, 1.0}
\definecolor{antiquewhite}{rgb}{0.98, 0.92, 0.84}

\newcolumntype{C}[1]{>{\centering\arraybackslash}p{#1}}

\hypersetup{
	hypertexnames=true, linkcolor=blue, anchorcolor=black,
	citecolor=blue, urlcolor=blue
}
\usepackage[multiple]{footmisc}

\usepackage{subcaption}

\newtheorem{myAttack}{Attack}

\begin{document}
	
\pagenumbering{arabic}
\newcommand{\protocol}{GuardML}
\newcommand{\twoPervPPML}{2GML}
\newcommand{\threePervPPML}{3GML}
\newcommand{\allPPML}{2GML/3GML}
\newcommand{\protocolTTT}{\texttt{GuardML}}
\newcommand{\twoPervPPMLTTT}{\texttt{2GML}}
\newcommand{\threePervPPMLTTT}{\texttt{3GML}}
\newcommand{\ecgPPML}{\texttt{ecgPPML}}

\newcommand{\algorithmautorefname}{Algorithm}
%

\title{GuardML: Efficient Privacy-Preserving Machine Learning Services Through Hybrid Homomorphic Encryption}
\renewcommand{\shorttitle}{GuardML}

\author{Eugene Frimpong}
\orcid{0000-0002-4924-5258}
\affiliation{%
  \institution{Tampere University}
  \streetaddress{Kalevantie 4}
  \city{Tampere}
  \country{Finland}
  \postcode{33100}  
}
\email{eugene.frimpong@tuni.fi}

\author{Khoa Nguyen}
\affiliation{%
  \institution{Tampere University}
  \streetaddress{Kalevantie 4}
  \city{Tampere}
  \country{Finland}
  \postcode{33100}  
}
\email{khoa.nguyen@tuni.fi}

\author{Mindaugas Budzys}
\affiliation{%
  \institution{Tampere University}
  \streetaddress{Kalevantie 4}
  \city{Tampere}
  \country{Finland}
  \postcode{33100} 
}
\email{mindaugas.budzys@tuni.fi}

\author{Tanveer Khan}
\affiliation{%
  \institution{Tampere University}
  \streetaddress{Kalevantie 4}
  \city{Tampere}
  \country{Finland}
  \postcode{33100} 
}
\email{tanveer.khan@tuni.fi}

\author{Antonis Michalas}
\affiliation{%
  \institution{Tampere University}
  \streetaddress{Kalevantie 4}
  \city{Tampere}
  \country{Finland}
  \postcode{33100} 
}
\email{antonios.michalas@tuni.fi}

\renewcommand{\shortauthors}{Frimpong et al.}
%
%
%
\begin{abstract}
Machine Learning (ML) has emerged as one of data science's most transformative and influential domains. However, the widespread adoption of ML introduces privacy-related concerns owing to the increasing number of malicious attacks targeting ML models. To address these concerns, Privacy-Preserving Machine Learning (PPML) methods have been introduced to safeguard the privacy and security of ML models. One such approach is the use of Homomorphic Encryption (HE). However, the significant drawbacks and inefficiencies of traditional HE render it impractical for highly scalable scenarios. Fortunately, a modern cryptographic scheme, Hybrid Homomorphic Encryption (HHE), has recently emerged, combining the strengths of symmetric cryptography and HE to surmount these challenges. Our work seeks to introduce HHE to ML by designing a PPML scheme tailored for end devices. We leverage HHE as the fundamental building block to enable secure learning of classification outcomes over encrypted data, all while preserving the privacy of the input data and ML model. We demonstrate the real-world applicability of our construction by developing and evaluating an HHE-based PPML application for classifying heart disease based on sensitive ECG data. Notably, our evaluations revealed a slight reduction in accuracy compared to inference on plaintext data. Additionally, both the analyst and end devices experience minimal communication and computation costs, underscoring the practical viability of our approach. The successful integration of HHE into PPML provides a glimpse into a more secure and privacy-conscious future for machine learning on relatively constrained end devices.

\end{abstract}
%
%
%


\maketitle

\section{Introduction}
\label{sec:introduction}

Machine Learning (ML) has increasingly become one of the most impactful fields of data science in recent years, allowing various users to classify and make predictions based on multi-dimensional data. One of the main metrics for ML is the accuracy of the prediction or classification results. However, to achieve this, the results should be accompanied by a large amount of high-quality training data requiring the collaboration of several organizations. Currently, regulations such as the General Data Protection Regulations (GDPR) forbid the sharing and processing of sensitive data without the data subject's consent. It has, therefore, become crucial to uphold data privacy and confidentiality when obtaining data from other organizations. One solution to this problem is employing Privacy-Preserving Machine Learning (PPML). PPML ensures that the use of data protects user privacy and that data is utilized in a safe fashion, avoiding leakage of confidential and private information. To this end, researchers have proposed and implemented various PPML-achieving techniques, varying from secure cryptographic schemes to distributed, hybrid, and data modification approaches. This work focuses on cryptographic approaches, the most commonly used one being Homomorphic Encryption (HE)~\cite{rivest1978data,gentry2009fully}, which exhibits a high potential in ML applications.

HE allows users to perform computations such as addition or multiplication on encrypted data~\cite{rivest1978data}. One of the first fully HE (FHE) schemes was proposed by C. Gentry~\cite{gentry2009fully}. FHE allows infinite operations on encrypted data while still producing a valid result. Since Gentry's seminal work~\cite{gentry2009fully}, multiple HE schemes have been proposed to improve the efficiency and applicability of HE. Currently, the most popular schemes used are \textbf{CKKS}~\cite{cheon2017homomorphic}, \textbf{TFHE}~\cite{chillotti2016faster} and \textbf{BFV}~\cite{brakerski2012fully,fan2012somewhat}. Because these schemes allow arbitrary computation on encrypted data, FHE has created exciting new applications in several areas, such as Machine Learning as a Service (MLaaS)\cite{lou2020glyph,lee2022privacy, khan2021blind, khan2023learning, khan2023love, nguyen2023split, khan2023more, khan2023split}. In MLaaS applications, data is encrypted under an HE scheme and transferred to the cloud for processing. Furthermore, due to the versatility of HE, it can also be used to keep the model private through model parameter encryption. Despite the various advances in HE, it has yet to find mainstream use because of high computational complexity and extended ciphertext expansion, resulting in very large ciphertexts. To address these issues, researchers have turned to the Hybrid Homomorphic Encryption (HHE) 
concept ~\cite{dobraunig2021pasta,bakas2022symmetrical}, 
which uses symmetric ciphers to make HE more accessible for users.

\noindent \textbf{\textit{HHE in a Nutshell:}} In an HHE scheme, rather than encrypt data with only an HE scheme, users locally encrypt data via a symmetric key encryption scheme. Subsequently, they homomorphically encrypt the symmetric key used in the encryption process with an HE scheme. The two ciphertexts resulting from the abovementioned encryptions are then forwarded to the server. Upon receiving both ciphertexts, the server transforms the symmetric ciphertext into a homomorphic ciphertext using the homomorphically encrypted symmetric key. The use of HHE produces ciphertexts of a substantially smaller size compared to using only HE schemes because of the symmetric encryption. The reduced size drastically lowers the communication overhead a user would typically incur using only a traditional HE approach. However, due to high multiplicative depth, not all symmetric key encryption schemes are compatible with HHE. To this end, researchers have designed a number of HE-friendly symmetric ciphers for use in HHE schemes, such as HERA/Rubato~\cite{cho2021transciphering,ha2022rubato}, Elisabeth~\cite{cosseron2022towards} and PASTA~\cite{dobraunig2021pasta}. Aside from reducing ciphertext size, 
HHE also provides a way for most consumer-grade devices to benefit from the advantages of HE schemes by transferring 
computationally expensive operations to the server or cloud.

For this work, we adopt the concept of HHE, as implemented in PASTA~\cite{dobraunig2021pasta}, to address the privacy problem of machine learning prediction as a service. 
Apart from the novelty of incorporating HHE in PPML, we aim to smooth out the big hurdles of implementing strong PPML models in a wide range of devices.

\smallskip

\noindent \textbf{\textit{Contributions:}} We provide a realistic solution that carefully considers the vagaries of PPML, all while exploring the use of new technologies that might unleash creative ideas in the decades to come. Additionally, the paper makes the following 
contributions: 
\begin{enumerate}[\textbf C1.]
        \item We first demonstrate the effective utilization of HHE to tackle the challenges of PPML. By extending the use of PASTA to ML applications, our approach introduces HHE as a key element in PPML, unlocking new possibilities. Our primary goal is to overcome the significant obstacles in implementing robust PPML models across diverse devices. 
	\item We present two formally designed protocols enabling an authorized entity (e.g., an analyst) to process encrypted data as if it were unencrypted efficiently. Rigorous security proofs demonstrate that our protocols preserve user privacy, ensuring no leakage that could compromise confidentiality.
	\item Through extensive experiments, we showcase the practicality of our protocol in a real-world ML scenario using a sensitive medical dataset. The experimental results indicate that our PPML protocol achieves nearly comparable accuracy to the plaintext version while safeguarding both the dataset and the neural network's privacy. Additionally, most of the computation cost is effectively outsourced to a CSP.
 

\end{enumerate}

\section{Related Works}
\label{sec:relatedwork}
\noindent\textbf{Homomorphic Encryption:} 
Various recent works have proposed using HE schemes in implementing PPML. Gentry's work~\cite{gentry2009fully} revolutionized the field of HE and paved the way for 
multiple modern schemes, such as TFHE~\cite{chillotti2020tfhe}, BFV~\cite{fan2012somewhat}, and CKKS~\cite{cheon2017homomorphic} in PPML applications. 
BFV was one of the first improvements to the original HE scheme proposed by Gentry and works by limiting expensive bootstrapping operations. BFV is, therefore, referred to as a Somewhat Homomorphic Encryption (SHE) scheme and allows a limited amount of operation to be performed on integer ciphertexts. Cheon~\textit{et al.} also proposed the CKKS scheme, which also allows a limited number of operations to be performed on ciphertext but allows computations on floating point data~\cite{cheon2017homomorphic}. Chillotti~\textit{et al.} then proposed the TFHE~\cite{chillotti2016faster} scheme. TFHE greatly improves the efficiency of bootstrapping and allows an unlimited amount of bitwise operations on binary data. Each of the aforementioned schemes has been applied to ML applications, requiring different techniques to implement non-linear activation functions. For example, TFHE relies on fast and efficient LUT searches to compute non-linear activations~\cite{lou2020glyph}, while BFV and CKKS require polynomial approximations~\cite{hesamifard2018privacy,lee2022privacy}. Each HE scheme has been applied to PPML with high-accuracy results. Examples of TFHE-based PPML works are TAPAS~\cite{sanyal2018tapas}, FHE-DiNN~\cite{bourse2018fast} and Glyph~\cite{lou2020glyph}. An example of a CKKS-based PPML protocol is POSEIDON~\cite{sav2020poseidon}, while an example of a BFV-based PPML work is HCNN~\cite{al2020towards}. HE schemes suffer from large ciphertext sizes and high computational complexities, which make them unsuitable for every environment. Enter HHE.

\smallskip
\noindent\textbf{Hybrid Homomorphic Encryption:} 
The first approaches to designing HHE schemes relied on existing and well-established symmetric ciphers such as AES~\cite{gentry2012homomorphic}. However, AES has been proven to not be a good fit for HHE schemes, primarily due to its large multiplicative depth~\cite{dobraunig2021pasta}. Thus, research in the field of HHE took a different approach, where the main focus shifted to the design of symmetric ciphers with different optimization criteria, 
such as eliminating the ciphertext expansion~\cite{canteaut2018stream} or using filter permutators~\cite{meaux2019improved}. However, to date, HHE has seen limited practical application~\cite{bakas2022symmetrical} in real-world applications, and only a handful of works exist in the field of PPML. To the best of our knowledge, the \textit{main} HHE schemes currently are HERA~\cite{cho2021transciphering}, Elisabeth~\cite{cosseron2022towards} and PASTA~\cite{dobraunig2021pasta}. The authors of HERA also proposed Rubato~\cite{ha2022rubato}; however, the specifications remain largely the same as in HERA. These proposed approaches have different specifications and can be applied to different use cases. HERA~\cite{cho2021transciphering} is a stream cipher based on the CKKS HE scheme and allows computations on floating point data types. In comparison, Elisabeth~\cite{cosseron2022towards} is designed to utilize the TFHE scheme, while PASTA~\cite{dobraunig2021pasta} is based on BFV for integer data types.

\begin{table}[ht!]
\scriptsize
\centering
\begin{tabular}{V{4}C{1.8cm} V{3}C{1.5cm}|C{1,5cm}|C{2cm}V{4}} 
 \hlineB{4}
 \rowcolor{Gray}
 \cellcolor{azure(web)(azuremist)} & \textbf{HERA}~\cite{cho2021transciphering} & \textbf{Elisabeth}~\cite{cosseron2022towards} & \textbf{PASTA}~\cite{dobraunig2021pasta} \\ 
 \hlineB{3}
 \cellcolor{Gray}\textbf{Programming Language} & Go & Rust & C \\ 
 \hline
 \cellcolor{Gray}\textbf{Data Type} & Floating point & Binary & Integer \\ 
 \hline
 \cellcolor{Gray}\textbf{HE Scheme} & CKKS & TFHE & BFV\\ 
 \hline
 \cellcolor{Gray}\textbf{Application in ML} & No & Yes & No \\
 \hline
 \cellcolor{Gray}\textbf{Defined over} & $\mathbb{Z}_\mathsf{q}$, where $\mathsf{q} > 2^{16}$ & $\mathbb{Z}_\mathsf{q}$, where $\mathsf{q} = 2^{4}$ &$\mathbb{F}_\mathsf{p}$, where $\mathsf{p}$ is a~16-bit prime \\ 
 \hline
 \cellcolor{Gray}\textbf{Security Level} & 80 or 128-bit & 128-bit & 128-bit \\ 
 \hline
 \cellcolor{Gray}\textbf{Quantization} & No & Yes & Yes \\ 
 \hlineB{3}
\end{tabular}
\caption{Comparison of different HHE schemes}
\label{tab:comp-hhe-schemes}
\end{table}

HERA and PASTA are defined over $\mathbb{Z}_\mathsf{q}$, where $\mathsf{q} = 2^{16} + 1$, and can store up to 16-bit inputs. Meanwhile, Elisabeth is defined over $\mathbb{Z}_\mathsf{q}$, where $\mathsf{q} = 2^4$, and can store up to 4 bits of data. Each approach achieves the same security level of~128-bits. Additionally, HERA also provides tests for a security level of 80 bits. As HERA allows computations on floating point data types, it does not require quantization on certain inputs. Elisabeth and PASTA, on the other hand, require quantization to operate on floating point numbers, which introduces a rounding error, which can reduce the accuracy of certain applications, i.e.\ ML. To the best of our knowledge, HERA has yet to be applied to an ML application. One of the use cases provided by Elisabeth is a CNN classification task on the Fashion-MNIST dataset and shows equivalent accuracy~($84.18\%$) to the cleartext model without HE. \autoref{tab:comp-hhe-schemes} provides an overview of the different popular HHE schemes. For this work, we aim to evaluate the suitability of PASTA when applied to PPML and investigate the feasibility of applying PASTA to various real-world ML datasets, particularly medical datasets, with the potential to achieve comparable accuracy and higher efficiency than conventional HE techniques.

\section{Preliminaries}
\label{sec: preliminaries}

\subsection{Homomorphic Encryption}
\label{subsec:he}
\begin{definition}[Homomorphic Encryption]
Let $\mathsf{HE}$ be a (public-key) homomorphic encryption scheme with a quadruple of PPT algorithms $\mathsf{HE = (KeyGen, } \allowbreak \mathsf{Enc, Dec, Eval)}$ such that:
\end{definition}
\begin{itemize}
\item $\mathbf{HE.KeyGen:}$ The key generation algorithm $\left(\mathsf{pk, evk, sk}\right) \leftarrow \mathsf{HE.KeyGen}(1^{\lambda})$ takes as input a unary representation of the security parameter $\lambda$, and outputs  a public key $\mathsf{pk}$, an evaluation key $\mathsf{evk}$ and a private key $\mathsf{sk}$.
\item $\mathbf{HE.Enc:}$ The encryption algorithm ${c \leftarrow \mathsf{HE.Enc}(\mathsf{pk}, x)}$ takes as input the public key $\mathsf{pk}$ and a message $x$ and outputs a ciphertext $c$.
\item $\mathbf{HE.Eval:}$ The algorithm $c_f \leftarrow \mathsf{HE.Eval}({\mathsf{evk}}, f, c_1, \dots, c_n)$ takes as input the evaluation key $\mathsf{evk}$, a function $f$, and a set of $n$ ciphertexts, and outputs a ciphertext $c_f$.
\item $\mathbf{HE.Dec:}$ The decryption algorithm $\mathsf{HE.Dec}({\mathsf{sk}}, c) \rightarrow x$, takes as input the secret key $\mathsf{sk}$ and a ciphertext $c$, and outputs a plaintext $x$.
\end{itemize}

\subsection{Hybrid Homomorphic Encryption}
\label{subsec: hhe}

\begin{definition}[Hybrid Homomorphic Encryption] Let $\mathsf{HE}$ be a Homomorphic Encryption scheme and $\mathsf{SKE} = (\mathsf{Gen, Enc, Dec})$ be a symmetric-key encryption scheme. Moreover, let $\mathcal{X} = (x_1, \dots, x_n)$ be the message space and $\lambda$ the security parameter. An $\mathsf{HHE}$ scheme 
consists of five PPT algorithms $\mathsf{HHE = (KeyGen, Enc, Decomp, Eval,}$ $\mathsf{Dec)}$ such that: 
\end{definition}
			 

		
			 

\begin{itemize}
		\item $\mathbf{HHE.KeyGen}$: The key generation algorithm takes as input a security parameter $\lambda$ and outputs a HE public/private key pair ($\mathsf{pk}$/$\mathsf{sk}$) and a HE evaluation key ($\mathsf{evk}$).
		
		\item $\mathbf{HHE.Enc}$: The encryption algorithm consists of three steps:
		\begin{itemize}
			\item $\mathsf{SKE.Gen}$: The SKE generation algorithm takes as input the security parameter $\lambda$ and outputs a symmetric key $\mathsf{K}$.

                \item $\mathsf{HE.Enc}$: An HE encryption algorithm that takes as input $\mathsf{pk}$ and $\mathsf{K}$, and outputs $c_\mathsf{K}$ -- a homomorphically encrypted representation of the symmetric key $\mathsf{K}$.
			\item $\mathsf{SKE.Enc}$: The SKE encryption algorithm takes as input a message $x$ and $\mathsf{K}$ and outputs a ciphertext $c$.
		\end{itemize}
		\item $\mathbf{HHE.Decomp}$: This algorithm takes as an input the evaluation key $\mathsf{evk}$, the symmetrically encrypted ciphertext $c$, and the homomorphically encrypted symmetric key $c_\mathsf{K}$, and outputs $c'$ -- a homomorphic encryption of the 
        message $x$.       
		\item $\mathbf{HHE.Eval}$: This algorithm takes as input $n$ homomorphic ciphertexts $c'_n$, where $n \geq 2$, the evaluation key $\mathsf{evk}$ and a homomorphic function $f$, and outputs a ciphertext $c'_{eval}$ of the evaluation results.
 
		\item $\mathbf{HHE.Dec}$: The decryption algorithm takes as input a private key $\mathsf{sk}$ and the evaluated ciphertext $c'_{eval}$ and outputs $f(x)$.  
	\end{itemize} 


The correctness of an HHE scheme follows directly from the correctness of the underlying public-key HE scheme. 


\subsection{Machine Learning}
\label{subsec:prelim_ml}
ML is a set of algorithms that leverage already-available training data 
as input to train a model $f(\boldsymbol{\theta})$. The process of training is to find the optimal parameters $\boldsymbol{\theta} = (w, b)$, where $w$ are weights matrices and $b$ are biases that can provide accurate predictions on the training data. Once trained, $f(\boldsymbol{\theta})$ can be used to provide predictions on unseen input data, which is called the ``inference'' or ``prediction'' phase. In this work, we leverage HHE to build PPML protocols that preserve data and model privacy in 
inference phase. 

\section{System Model}
\label{sec:architecture}
In this section, we introduce our system model by explicitly describing our protocol's 
main entities and 
their capabilities. 


\begin{itemize}
	\item \textbf{User}: Let $\mathcal{U} = \{u_1, \ldots, u_n\}$ be the set of all users. Each user generates a unique symmetric key $\mathsf{K_i}$ locally and encrypts their data. Then the generated ciphertexts are outsourced to the $\mathbf{CSP}$ along with an HE encryption $c_\mathsf{K_i}$ of the underlying symmetric key.

	\item \textbf{Cloud Service Provider (CSP)}: 
 Primarily responsible for gathering symmetrically encrypted data from multiple users. The $\mathbf{CSP}$ is tasked with converting the symmetrically encrypted data into homomorphic ciphertexts and, upon request, performing blind operations on them. 
    
	\item \textbf{Analyst (A)}: In one of our proposed protocols, there exists an analyst who owns an ML model and is interested in learning the output of ML operations on the encrypted data stored at the $\mathbf{CSP}$. In this protocol, $\mathbf{A}$ decrypts the encrypted data from the HE evaluation of collected encrypted data and, thus, may gain insights from user data.
\end{itemize}

		

\section{GuardML}
\label{sec:protocol}
In this section, we discuss in detail the construction of \protocolTTT{} -- our Hybrid Homomorphic Privacy-Preserving protocols that constitute the core of this paper. \protocolTTT{} is comprised of two protocols, \twoPervPPMLTTT{} and \threePervPPMLTTT{}. The primary differences between the protocols mentioned above lie in the presence of an \textbf{analyst} and support for a \textbf{multi-user} scenario. Each protocol may be suitable for different use cases. For example, \twoPervPPMLTTT{} can be used in commercial ML applications where the \textbf{CSP}, such as AWS, Azure, etc., owns the ML models. While \threePervPPMLTTT, on the other hand, would be ideal for a setting where the analyst is 
owner of the 
model and does not wish to reveal the contents to 
\textbf{CSP}. In this setting, the \textbf{CSP}'s role is reduced to simply performing operations on encrypted data and models.   



\medskip
\noindent\textbf{\textit{Building Blocks:}} 
\label{subsubsec:buildingblocks}
Before proceeding to describe each protocol, we first define the building blocks 
used in our constructions. 

\begin{itemize}
	\item A secure symmetric cipher $\mathsf{SKE = (Gen, Enc, Dec)}$.
	\item A BFV-based HHE scheme $\mathsf{HHE = (KeyGen, Enc, Dec, Deco}$ $\mathsf{mp, Eval)}$.
\end{itemize}

Additionally, to provide secure communication, 
we define a public key encryption scheme, which supports message encryption and decryption, a signature scheme used for message signing and verification, and a secure cryptographic hash function to verify message integrity. We make the following assumptions: 

\begin{itemize}
	\item A CCA2 secure public-key encryption scheme $\mathsf{PKE = (Gen,}$ $\mathsf{Enc, Dec)}$.
	\item An EUF-CMA secure signature scheme $\mathsf{\sigma = (sign, ver)}$.    
	\item A first and second pre-image resistant cryptographic hash function $\mathsf{H(\cdot)}$.
\end{itemize}



\subsection{GuardML: 2-Party Setting}
\label{subsec:2party}
\noindent \textbf{\textit{High-Level Overview:}} \enskip The first version of \texttt{GuardML} is \twoPervPPMLTTT{} -- a 2-party protocol that consists of a $\mathbf{CSP}$ and a user $u_{i}$. 
In this setting, we assume that the $\mathbf{CSP}$ is the owner of a trained ML model with parameters $(w, b)$, while $u_{i}$ provides the input data $x_i$ to the model. Initially, the user $u_i$ generates the necessary HHE keys ($\mathsf{pk_{u_i}, sk_{u_i}, evk_{u_i}})$. 
It then publishes the public key $\mathsf{pk_{u_i}}$ and sends the evaluation key $\mathsf{evk_{u_i}}$ to the $\mathbf{CSP}$. Subsequently, $u_{i}$ generates a unique symmetric key $\mathsf{K_{i}}$. On completing the key generation phase, $u_{i}$ begins the data upload phase by first generating $c_\mathsf{K_i}$, a homomorphic encryption of the symmetric key $\mathsf{K_{i}}$, and then calculates a symmetric encryption of the data $x_{i}$ with $\mathsf{K_{i}}$. The data upload phase concludes with the sending of both ciphertext values to the $\mathbf{CSP}$. 
Upon reception, the $\mathbf{CSP}$ begins the secure evaluation phase by first transforming the symmetrically encrypted data $c_{x_i}$ into a homomorphic ciphertext $c'_{x_i}$. On successful run, $\mathbf{CSP}$ uses $\mathsf{evk_{u_i}}$, $c'_{x_i}$ and $f$ to produce an encrypted prediction $c_{res}$. The encrypted result is then sent to $u_i$ for decryption. 
\medskip

\noindent\textbf{\textit{Formal Construction:}} 
\label{subsubsec:twopartyformal}
\autoref{fig:twoparty} provides an overview of the \twoPervPPMLTTT{} protocol, which is divided into four distinct phases, namely \texttt{\twoPervPPML.Se} \texttt{tup}, \texttt{\twoPervPPML.Upload}, \texttt{\twoPervPPML.Eval}, and \texttt{\twoPervPPML.Classify}. 

\smallskip

\noindent\framebox{$\mathsf{\mathbf{\twoPervPPML.Setup}}$:}
In the setup phase, both parties generate their respective signing/verification key pairs for the signature scheme $\sigma$ and publish the verification keys. The \textbf{CSP} runs the $\mathsf{PKE.Gen}$ algorithm 
to generate a public/private key pair $(\mathsf{pk_{CSP}, sk_{CSP}})$. On the other hand, $u_{i}$ runs the $\mathsf{HHE.KeyGen}$ algorithm 
to generate the public, private and evaluation keys $(\mathsf{pk_{u_i}, sk_{u_i}, evk_{u_i}})$ used for HHE operations. Finally, $u_{i}$ publishes $\mathsf{pk_{u_i}}$ and outsources $\mathsf{evk_{u_i}}$ to the \textbf{CSP} through \textit{$m_{1}$}.
\begin{equation}
	m_{1} = \langle t_{1}, \mathsf{Enc(pk_{CSP}, \mathsf{evk_{u_i}})}, \sigma_{u_i}(\mathsf{H}(t_{1} || \mathsf{evk_{u_i}}))\rangle , \nonumber
\end{equation}
where $\sigma_{u_i}$ is a signature created by $u_{i}$. Upon reception, the \textbf{CSP} verifies the signature 
by using $u_{i}$’s verification key and the freshness of the message through the timestamp. If the verification fails, the \textbf{CSP} aborts the protocol and outputs $\perp$. Otherwise, the \textbf{CSP} stores $\mathsf{evk_{u_i}}$ locally.

\smallskip

\noindent\framebox{$\mathsf{\mathbf{\twoPervPPML.Upload}}$:}
In this phase, $u_{i}$ runs the $\mathsf{HHE.Enc}$ algorithm to generate a symmetric encryption key $\mathsf{K_{i}}$ through $\mathsf{SKE.Gen}$ and then encrypts the plaintext data $x_i$ into ciphertext $c_{x_i}$ using the $\mathsf{SKE.Enc}$ algorithm, which takes $(\mathsf{K_i}, x_i)$ as an input. After encrypting the data, $u_{i}$ also runs $\mathsf{HE.Enc}$ to homomorphically encrypt $\mathsf{K_i}$ into $c_{\mathsf{K_i}}$ with their $\mathsf{pk_{u_i}}$. After both encryptions are finished $u_i$ sends both ciphertexts to the \textbf{CSP} with \textit{${m_{2}}$}: 
\begin{equation}
	m_{2} = \langle t_{2}, c_{x_i}, c_{\mathsf{K_i}}, \sigma_{u_i}(\mathsf{H}(t_{2} || c_{x_i} || c_{\mathsf{K_i}}))\rangle , \nonumber
\end{equation}
On receiving $m_2$, the \textbf{CSP} verifies the signature by using $u_{i}$’s verification key and the freshness of the message through the timestamp. If the verification fails, the \textbf{CSP} aborts the protocol and outputs $\perp$. Otherwise, the \textbf{CSP} continues to the secure evaluation phase. 

\smallskip

\noindent\framebox{$\mathsf{\mathbf{\twoPervPPML.Eval}}$:}
The secure evaluation phase begins 
with the \textbf{CSP} transforming the received symmetric ciphertext $c_{x_i}$ into HE ciphertext $c'_{x_i}$ by running the $\mathsf{HHE.Decomp}$ algorithm. $\mathsf{HHE.Decomp}$ uses $\mathsf{HE.Eval}$ which takes as an input ($\mathsf{evk_{u_i}}, \mathsf{SKE.Dec}, c_\mathsf{K_i}, c_{x_i}$), where $\mathsf{SKE.Dec}$ is the symmetric cipher decryption algorithm to transform the ciphertext. Afterwards, the \textbf{CSP} takes $c'_{x_i}$ and their ML model parameters $(w, b)$ and inputs both of them into the $\mathsf{HHE.Eval}$ along with the $\mathsf{evk_{u_i}}$ to compute an encrypted prediction $c_{res}$ from the encrypted data. Finally, the \textbf{CSP} securely sends the encrypted result back to $u_i$ via $m_{3}$: 
\begin{equation}
	m_{3} = \langle t_{3}, c_{res}, \sigma_{CSP}(\mathsf{H}(t_{3} || c_{res}))\rangle , \nonumber
\end{equation}
Upon reception, $u_{i}$ verifies the integrity and the freshness of the message. 
If the verification fails, $u_i$ aborts the protocol and outputs $\perp$. Otherwise, 
continues to the final phase.

\smallskip 

\noindent\framebox{$\mathsf{\mathbf{\twoPervPPML.Classify}}$:}
In the final phase of the protocol, $u_i$ decrypts the received ciphertext to gain insight into the data. This is done by $u_{i}$ running the $\mathsf{HHE.Dec}$ algorithm with inputs $\mathsf{sk_{u_i}}, c_{res}$ and outputs the prediction $res$. 

\begin{figure}[ht!]
	\centering
	\resizebox{\linewidth}{!}{%
		\tikzset{every picture/.style={line width=0.75pt}}        
		\begin{tikzpicture}[x=0.75pt,y=0.75pt,yscale=-1,xscale=1]
			\draw    (97,411) -- (225.5,411) -- (421.5,411) ;
			\draw [shift={(424.5,411)}, rotate = 180] [fill={rgb, 255:red, 0; green, 0; blue, 0 }  ][line width=0.08]  [draw opacity=0] (8.93,-4.29) -- (0,0) -- (8.93,4.29) -- cycle    ;
			
			\draw    (100.5,693) -- (425.5,693) ;
			\draw [shift={(97.5,693)}, rotate = 0] [fill={rgb, 255:red, 0; green, 0; blue, 0 }  ][line width=0.08]  [draw opacity=0] (8.93,-4.29) -- (0,0) -- (8.93,4.29) -- cycle    ;
			\draw [color={rgb, 255:red, 0; green, 0; blue, 0 }  ,draw opacity=1 ][line width=1.5]  [dash pattern={on 1.69pt off 2.76pt}]  (48.5,272) -- (517.5,272) ;
			\draw  [dash pattern={on 4.5pt off 4.5pt}]  (87.5,249) -- (87.5,294) ;
			\draw  [dash pattern={on 4.5pt off 4.5pt}]  (432.5,245) -- (434.5,412) ;
			\draw   (56,213) -- (124,213) -- (124,248.5) -- (56,248.5) -- cycle ;
			\draw   (400,210) -- (467.06,210) -- (467.06,245.5) -- (400,245.5) -- cycle ;
			\draw   (79,292) -- (96.5,292) -- (96.5,811) -- (79,811) -- cycle ;
			\draw   (426,408) -- (443.5,408) -- (443.5,810) -- (426,810) -- cycle ;
			\draw    (97,298) -- (158.5,298) ;
			\draw    (158.5,298) -- (158.5,319) ;
			\draw    (112.5,319) -- (158.5,319) ;
			\draw [shift={(109.5,319)}, rotate = 0] [fill={rgb, 255:red, 0; green, 0; blue, 0 }  ][line width=0.08]  [draw opacity=0] (8.93,-4.29) -- (0,0) -- (8.93,4.29) -- cycle    ;
			\draw   (96.5,319) -- (109.5,319) -- (109.5,352) -- (96.5,352) -- cycle ;
			
			\draw    (97,419) -- (158.5,419) ;
			
			\draw    (158.5,419) -- (158.5,440) ;
			
			\draw    (112.5,440) -- (158.5,440) ;
			\draw [shift={(109.5,440)}, rotate = 0] [fill={rgb, 255:red, 0; green, 0; blue, 0 }  ][line width=0.08]  [draw opacity=0] (8.93,-4.29) -- (0,0) -- (8.93,4.29) -- cycle    ;
			
			\draw   (96.5,440) -- (109.5,440) -- (109.5,466) -- (96.5,466) -- cycle ;
			\draw  [dash pattern={on 0.84pt off 2.51pt}]  (108.5,352) -- (160.5,352) ;
			\draw  [dash pattern={on 0.84pt off 2.51pt}]  (160.5,352) -- (160.5,373) ;
			
			\draw  [dash pattern={on 0.84pt off 2.51pt}]  (100.5,372) -- (159.5,372) ;
			\draw [shift={(97.5,372)}, rotate = 0] [fill={rgb, 255:red, 0; green, 0; blue, 0 }  ][line width=0.08]  [draw opacity=0] (8.93,-4.29) -- (0,0) -- (8.93,4.29) -- cycle    ;
			\draw    (97,516) -- (225.5,516) -- (421.5,516) ;
			\draw [shift={(424.5,516)}, rotate = 180] [fill={rgb, 255:red, 0; green, 0; blue, 0 }  ][line width=0.08]  [draw opacity=0] (8.93,-4.29) -- (0,0) -- (8.93,4.29) -- cycle    ;
			\draw  [dash pattern={on 0.84pt off 2.51pt}]  (108.5,466) -- (160.5,466) ;
			\draw  [dash pattern={on 0.84pt off 2.51pt}]  (160.5,466) -- (160.5,487) ;
			\draw  [dash pattern={on 0.84pt off 2.51pt}]  (100.5,486) -- (159.5,486) ;
			\draw [shift={(97.5,486)}, rotate = 0] [fill={rgb, 255:red, 0; green, 0; blue, 0 }  ][line width=0.08]  [draw opacity=0] (8.93,-4.29) -- (0,0) -- (8.93,4.29) -- cycle    ;
			
			\draw    (444,532) -- (505.5,532) ;
			\draw    (505.5,532) -- (505.5,553) ; 
			\draw    (459.5,553) -- (505.5,553) ;
			\draw [shift={(456.5,553)}, rotate = 0] [fill={rgb, 255:red, 0; green, 0; blue, 0 }  ][line width=0.08]  [draw opacity=0] (8.93,-4.29) -- (0,0) -- (8.93,4.29) -- cycle    ;
			\draw   (443.5,553) -- (456.5,553) -- (456.5,579) -- (443.5,579) -- cycle ;
			\draw  [dash pattern={on 0.84pt off 2.51pt}]  (455.5,579) -- (507.5,579) ;
			\draw  [dash pattern={on 0.84pt off 2.51pt}]  (507.5,579) -- (507.5,600) ;
			\draw  [dash pattern={on 0.84pt off 2.51pt}]  (447.5,599) -- (506.5,599) ;
			\draw [shift={(444.5,599)}, rotate = 0] [fill={rgb, 255:red, 0; green, 0; blue, 0 }  ][line width=0.08]  [draw opacity=0] (8.93,-4.29) -- (0,0) -- (8.93,4.29) -- cycle    ;
			\draw    (444,620) -- (505.5,620) ;
			\draw    (505.5,620) -- (505.5,641) ;
			\draw    (459.5,641) -- (505.5,641) ;
			\draw [shift={(456.5,641)}, rotate = 0] [fill={rgb, 255:red, 0; green, 0; blue, 0 }  ][line width=0.08]  [draw opacity=0] (8.93,-4.29) -- (0,0) -- (8.93,4.29) -- cycle    ;
			\draw   (443.5,641) -- (456.5,641) -- (456.5,667) -- (443.5,667) -- cycle ;
			\draw  [dash pattern={on 0.84pt off 2.51pt}]  (455.5,667) -- (507.5,667) ;
			\draw  [dash pattern={on 0.84pt off 2.51pt}]  (507.5,667) -- (507.5,688) ;
			\draw  [dash pattern={on 0.84pt off 2.51pt}]  (447.5,687) -- (506.5,687) ;
			\draw [shift={(444.5,687)}, rotate = 0] [fill={rgb, 255:red, 0; green, 0; blue, 0 }  ][line width=0.08]  [draw opacity=0] (8.93,-4.29) -- (0,0) -- (8.93,4.29) -- cycle    ;
			\draw    (97,702) -- (158.5,702) ;
			\draw    (158.5,702) -- (158.5,723) ;
			\draw    (112.5,723) -- (158.5,723) ;
			\draw [shift={(109.5,723)}, rotate = 0] [fill={rgb, 255:red, 0; green, 0; blue, 0 }  ][line width=0.08]  [draw opacity=0] (8.93,-4.29) -- (0,0) -- (8.93,4.29) -- cycle    ;
			\draw   (96.5,723) -- (109.5,723) -- (109.5,749) -- (96.5,749) -- cycle ;
			\draw  [dash pattern={on 0.84pt off 2.51pt}]  (108.5,749) -- (160.5,749) ;
			\draw  [dash pattern={on 0.84pt off 2.51pt}]  (160.5,749) -- (160.5,770) ;
			\draw  [dash pattern={on 0.84pt off 2.51pt}]  (100.5,769) -- (159.5,769) ;
			\draw [shift={(97.5,769)}, rotate = 0] [fill={rgb, 255:red, 0; green, 0; blue, 0 }  ][line width=0.08]  [draw opacity=0] (8.93,-4.29) -- (0,0) -- (8.93,4.29) -- cycle    ;
			\draw  [dash pattern={on 0.84pt off 2.51pt}]  (30.5,292) -- (79,292) ;
			\draw  [dash pattern={on 0.84pt off 2.51pt}]  (30.5,413) -- (79,413) ;
			\draw  [dash pattern={on 0.84pt off 2.51pt}]  (26.5,517) -- (75,517) ;
			\draw  [dash pattern={on 0.84pt off 2.51pt}]  (29.5,693) -- (78,693) ;
			\draw  [dash pattern={on 0.84pt off 2.51pt}]  (30.5,811) -- (79,811) ;
			\draw (418,222) node [anchor=north west][inner sep=0.75pt]   [align=left] {\textbf{CSP}};
			\draw (75,224) node [anchor=north west][inner sep=0.75pt]   [align=left] {\textbf{User}};
			
			\draw (163,301) node [anchor=north west][inner sep=0.75pt]   [align=left] {Run$\displaystyle \ \mathsf{HHE.KeyGen}$};
			\draw (164,353) node [anchor=north west][inner sep=0.75pt]   [align=left] {$\displaystyle (\mathsf{pk_{u_i} ,sk_{u_i} ,evk_{u_i}}) \ $};
			
			\draw (135,385.4) node [anchor=north west][inner sep=0.75pt]    {$m_{1} =\ \langle t_{1} ,\ \mathsf{Enc( pk_{CSP},\ evk_{u_i})} ,\ \sigma _{u_i}( \mathsf{H}( t_{1} \ || \ \mathsf{evk_{u_i}})) \rangle $};
			\draw (146,492.4) node [anchor=north west][inner sep=0.75pt]    {$m_{2} =\ \langle t_{2} ,\ c_{x_i},\ c_{\mathsf{K_i}} ,\ \sigma _{u_i}( \mathsf{H}( t_{2} \ || \ c_{x_i}\ || \ c_{\mathsf{K_i}})) \rangle $};
			\draw (163,423) node [anchor=north west][inner sep=0.75pt]   [align=left] {Run$\displaystyle \ \mathsf{HHE.Enc}$};
			\draw (164,469) node [anchor=north west][inner sep=0.75pt]   [align=left] {$\displaystyle \mathsf{K_i} ,\ c_{x_i},\ c_\mathsf{K_i}$};
			\draw (507.5,535) node [anchor=north west][inner sep=0.75pt]   [align=left] {Run$\displaystyle \ \mathsf{HHE.Decomp}$};
			\draw (507.5,623) node [anchor=north west][inner sep=0.75pt]   [align=left] {Run$\displaystyle \ \mathsf{HHE.Eval}$};
			\draw (513,672) node [anchor=north west][inner sep=0.75pt]   [align=left] {$\displaystyle c_{res}$};
			\draw (513,579) node [anchor=north west][inner sep=0.75pt]   [align=left] {$\displaystyle c'_{x_i}$};
			\draw (156,665.4) node [anchor=north west][inner sep=0.75pt]    {$m_{3} =\ \langle t_{3} ,\ c_{res} ,\ \sigma _{CSP}( \mathsf{H}( t_{3} \ || \ c_{res} ))\rangle $};
			\draw (163,706) node [anchor=north west][inner sep=0.75pt]   [align=left] {Run$\displaystyle \ \mathsf{HHE.Dec}$};
			\draw (164,755) node [anchor=north west][inner sep=0.75pt]   [align=left] {$\displaystyle{res}$};
			
			\draw (37.19,370.44) node [anchor=north west][inner sep=0.75pt]  [rotate=-270.41] [align=left] {\textbf{Setup}};
			
			\draw (37.19,490.44) node [anchor=north west][inner sep=0.75pt]  [rotate=-270.41] [align=left] {\textbf{Upload}};
			
			\draw (37.19,610.44) node [anchor=north west][inner sep=0.75pt]  [rotate=-270.41] [align=left] {\textbf{Eval}};
			\draw (37.19,780.44) node [anchor=north west][inner sep=0.75pt]  [rotate=-270.41] [align=left] {\textbf{Classify}};
	\end{tikzpicture}}
	\caption{$\mathsf{\twoPervPPML}$}
	\label{fig:twoparty}
\end{figure}

\subsection{GuardML: 3-Party Setting}
\label{subsec:3party}
\noindent \textbf{\textit{High-Level Overview:}} \enskip We note that \twoPervPPMLTTT{} is unsuitable for multi-user scenarios, where data is collected from multiple users and stored at the $\mathbf{CSP}$. In such a scenario, each user would be required to generate a unique set of HHE keys, while the $\mathbf{CSP}$ would need to store the $\mathsf{evk_{u_i}}$ of each user to run the protocol successfully. To resolve this, we present \threePervPPMLTTT{} -- an extended version of \twoPervPPMLTTT{} for the multi-client model. 
\threePervPPMLTTT{} consists of three parties: a set of users $\mathcal{U}$, \textbf{CSP}, and an analyst \textbf{A}. The protocol's steps are primarily the same as in \twoPervPPMLTTT{} (\autoref{subsec:2party}). The fundamental distinction is that under the three-party setting, the HHE keys are generated by \textbf{A} instead of $u_{i}$. In the \threePervPPMLTTT{} setting, we assume that an analyst \textbf{A} owns the ML model with parameters $(w, b)$, while a user $u_i$ provides the data $x_i$. \textbf{A} generates the required HHE keys ($\mathsf{pk_{A}, sk_{A}, evk_{A}}$), publishes $\mathsf{pk_{A}}$, and sends $\mathsf{evk_{A}}$ to the \textbf{CSP}. Each $u_{i}$ generates a symmetric key $\mathsf{K_{i}}$ and encrypts their data $x_{i}$ locally to output a symmetric ciphertext $c_{x_i}$. Additionally, $u_{i}$ also generates $c_{\mathsf{K_{i}}}$ -- a homomorphic encryption of the symmetric key $\mathsf{K_{i}}$ using \textbf{A}'s public key and sends both the encrypted data $c_{x_i}$ and $c_{\mathsf{K_{i}}}$ to the \textbf{CSP}. Upon reception, the \textbf{CSP} stores the values locally. In the evaluation phase, \textbf{A} can request a prediction to be performed on the stored $c'_{x_i}$ by first sending to the \textbf{CSP} their homomorphically encrypted pre-trained ML model parameters $(c_{w}, c_{b})$. On receiving the request from \textbf{A}, the \textbf{CSP} transforms $c_{x_i}$ into homomorphic cihpertext $c'_{x_i}$. \textbf{CSP} produces an encrypted result $c_{res}$ and sends the results back to \textbf{A}. Finally, \textbf{A} decrypts the prediction result $res$ using $\mathsf{sk_A}$.

\begin{figure}[ht!]
        \centering
	\resizebox{\linewidth}{!}{%
		\tikzset{every picture/.style={line width=0.75pt}} 
		\begin{tikzpicture}[x=0.75pt,y=0.75pt,yscale=-1,xscale=1]
			\draw    (117,431) -- (245.5,431) -- (441.5,431) ;
			\draw [shift={(444.5,431)}, rotate = 180] [fill={rgb, 255:red, 0; green, 0; blue, 0 }  ][line width=0.08]  [draw opacity=0] (8.93,-4.29) -- (0,0) -- (8.93,4.29) -- cycle    ;
			\draw    (120.5,713) -- (445.5,713) ;
			\draw [shift={(117.5,713)}, rotate = 0] [fill={rgb, 255:red, 0; green, 0; blue, 0 }  ][line width=0.08]  [draw opacity=0] (8.93,-4.29) -- (0,0) -- (8.93,4.29) -- cycle    ;
			\draw [color={rgb, 255:red, 0; green, 0; blue, 0 }  ,draw opacity=1 ][line width=1.5]  [dash pattern={on 1.69pt off 2.76pt}]  (68.5,292) -- (424,292) -- (805.5,291) ;
			\draw  [dash pattern={on 4.5pt off 4.5pt}]  (107.5,269) -- (107.5,314) ;
			\draw  [dash pattern={on 4.5pt off 4.5pt}]  (452.5,265) -- (454.5,432) ;
			\draw   (76,233) -- (144,233) -- (144,268.5) -- (76,268.5) -- cycle ; 
			\draw   (420,230) -- (487.06,230) -- (487.06,265.5) -- (420,265.5) -- cycle ;
			\draw   (99,312) -- (116.5,312) -- (116.5,831) -- (99,831) -- cycle ;
			\draw   (446,428) -- (463.5,428) -- (463.5,830) -- (446,830) -- cycle ;
			\draw    (117,318) -- (178.5,318) ;
			\draw    (178.5,318) -- (178.5,339) ;
			\draw    (132.5,339) -- (178.5,339) ;
			\draw [shift={(129.5,339)}, rotate = 0] [fill={rgb, 255:red, 0; green, 0; blue, 0 }  ][line width=0.08]  [draw opacity=0] (8.93,-4.29) -- (0,0) -- (8.93,4.29) -- cycle    ;
			\draw   (116.5,339) -- (129.5,339) -- (129.5,372) -- (116.5,372) -- cycle ; 
			\draw    (763,435) -- (824.5,435) ;
			\draw    (824.5,435) -- (824.5,456) ; 
			\draw    (778.5,456) -- (824.5,456) ;
			\draw [shift={(775.5,456)}, rotate = 0] [fill={rgb, 255:red, 0; green, 0; blue, 0 }  ][line width=0.08]  [draw opacity=0] (8.93,-4.29) -- (0,0) -- (8.93,4.29) -- cycle    ;
			\draw   (762.5,456) -- (775.5,456) -- (775.5,482) -- (762.5,482) -- cycle ;
			\draw  [dash pattern={on 0.84pt off 2.51pt}]  (128.5,372) -- (180.5,372) ;
			\draw  [dash pattern={on 0.84pt off 2.51pt}]  (180.5,372) -- (180.5,393) ;
			\draw  [dash pattern={on 0.84pt off 2.51pt}]  (120.5,392) -- (179.5,392) ;
			\draw [shift={(117.5,392)}, rotate = 0] [fill={rgb, 255:red, 0; green, 0; blue, 0 }  ][line width=0.08]  [draw opacity=0] (8.93,-4.29) -- (0,0) -- (8.93,4.29) -- cycle    ;
			\draw  [dash pattern={on 0.84pt off 2.51pt}]  (774.5,482) -- (826.5,482) ;
			\draw  [dash pattern={on 0.84pt off 2.51pt}]  (826.5,482) -- (826.5,503) ;
			\draw  [dash pattern={on 0.84pt off 2.51pt}]  (766.5,502) -- (825.5,502) ;
			\draw [shift={(763.5,502)}, rotate = 0] [fill={rgb, 255:red, 0; green, 0; blue, 0 }  ][line width=0.08]  [draw opacity=0] (8.93,-4.29) -- (0,0) -- (8.93,4.29) -- cycle    ;
			\draw    (464,552) -- (525.5,552) ;
			\draw    (525.5,552) -- (525.5,573) ;
			\draw    (479.5,573) -- (525.5,573) ;
			\draw [shift={(476.5,573)}, rotate = 0] [fill={rgb, 255:red, 0; green, 0; blue, 0 }  ][line width=0.08]  [draw opacity=0] (8.93,-4.29) -- (0,0) -- (8.93,4.29) -- cycle    ;
			\draw   (463.5,573) -- (476.5,573) -- (476.5,599) -- (463.5,599) -- cycle ;
			\draw  [dash pattern={on 0.84pt off 2.51pt}]  (475.5,599) -- (527.5,599) ;
			\draw  [dash pattern={on 0.84pt off 2.51pt}]  (527.5,599) -- (527.5,620) ;
			\draw  [dash pattern={on 0.84pt off 2.51pt}]  (467.5,619) -- (526.5,619) ;
			\draw [shift={(464.5,619)}, rotate = 0] [fill={rgb, 255:red, 0; green, 0; blue, 0 }  ][line width=0.08]  [draw opacity=0] (8.93,-4.29) -- (0,0) -- (8.93,4.29) -- cycle    ;
			\draw    (464,640) -- (525.5,640) ;
			\draw    (525.5,640) -- (525.5,661) ;
			\draw    (479.5,661) -- (525.5,661) ;
			\draw [shift={(476.5,661)}, rotate = 0] [fill={rgb, 255:red, 0; green, 0; blue, 0 }  ][line width=0.08]  [draw opacity=0] (8.93,-4.29) -- (0,0) -- (8.93,4.29) -- cycle    ;
			\draw   (463.5,661) -- (476.5,661) -- (476.5,687) -- (463.5,687) -- cycle ;
			\draw  [dash pattern={on 0.84pt off 2.51pt}]  (475.5,687) -- (527.5,687) ; 
			\draw  [dash pattern={on 0.84pt off 2.51pt}]  (527.5,687) -- (527.5,708) ;
			\draw  [dash pattern={on 0.84pt off 2.51pt}]  (467.5,707) -- (526.5,707) ;
			\draw [shift={(464.5,707)}, rotate = 0] [fill={rgb, 255:red, 0; green, 0; blue, 0 }  ][line width=0.08]  [draw opacity=0] (8.93,-4.29) -- (0,0) -- (8.93,4.29) -- cycle    ;
			
			\draw    (117,722) -- (178.5,722) ;
			\draw    (178.5,722) -- (178.5,743) ;
			\draw    (132.5,743) -- (178.5,743) ;
			\draw [shift={(129.5,743)}, rotate = 0] [fill={rgb, 255:red, 0; green, 0; blue, 0 }  ][line width=0.08]  [draw opacity=0] (8.93,-4.29) -- (0,0) -- (8.93,4.29) -- cycle    ;
			\draw   (116.5,743) -- (129.5,743) -- (129.5,769) -- (116.5,769) -- cycle ;
			\draw  [dash pattern={on 0.84pt off 2.51pt}]  (128.5,769) -- (180.5,769) ;
			\draw  [dash pattern={on 0.84pt off 2.51pt}]  (180.5,769) -- (180.5,790) ;
			\draw  [dash pattern={on 0.84pt off 2.51pt}]  (120.5,789) -- (179.5,789) ;
			\draw [shift={(117.5,789)}, rotate = 0] [fill={rgb, 255:red, 0; green, 0; blue, 0 }  ][line width=0.08]  [draw opacity=0] (8.93,-4.29) -- (0,0) -- (8.93,4.29) -- cycle    ;
			\draw  [dash pattern={on 0.84pt off 2.51pt}]  (50.5,312) -- (99,312) ;
			\draw  [dash pattern={on 0.84pt off 2.51pt}]  (50.5,433) -- (99,433) ; 
			\draw  [dash pattern={on 0.84pt off 2.51pt}]  (50.5,538) -- (99,538) ;
			\draw  [dash pattern={on 0.84pt off 2.51pt}]  (50.5,713) -- (99,713) ;
			\draw  [dash pattern={on 0.84pt off 2.51pt}]  (50.5,831) -- (99,831) ;
			\draw   (718,233) -- (785.06,233) -- (785.06,268.5) -- (718,268.5) -- cycle ;
			\draw  [dash pattern={on 4.5pt off 4.5pt}]  (753.5,269) -- (753.5,314) ;
			\draw   (745,312) -- (762.5,312) -- (762.5,831) -- (745,831) -- cycle ;
			\draw    (467.5,527) -- (744.5,527) ;
			\draw [shift={(464.5,527)}, rotate = 0] [fill={rgb, 255:red, 0; green, 0; blue, 0 }  ][line width=0.08]  [draw opacity=0] (8.93,-4.29) -- (0,0) -- (8.93,4.29) -- cycle    ;
			\draw    (117,633) -- (245.5,633) -- (441.5,633) ;
			\draw [shift={(444.5,633)}, rotate = 180] [fill={rgb, 255:red, 0; green, 0; blue, 0 }  ][line width=0.08]  [draw opacity=0] (8.93,-4.29) -- (0,0) -- (8.93,4.29) -- cycle    ; 
			\draw    (117,538) -- (178.5,538) ; 
			\draw    (178.5,538) -- (178.5,559) ;
			\draw    (132.5,559) -- (178.5,559) ;
			\draw [shift={(129.5,559)}, rotate = 0] [fill={rgb, 255:red, 0; green, 0; blue, 0 }  ][line width=0.08]  [draw opacity=0] (8.93,-4.29) -- (0,0) -- (8.93,4.29) -- cycle    ;
			\draw   (116.5,559) -- (129.5,559) -- (129.5,585) -- (116.5,585) -- cycle ;
			\draw  [dash pattern={on 0.84pt off 2.51pt}]  (128.5,585) -- (180.5,585) ;
			\draw  [dash pattern={on 0.84pt off 2.51pt}]  (180.5,585) -- (180.5,606) ;
			\draw  [dash pattern={on 0.84pt off 2.51pt}]  (120.5,605) -- (179.5,605) ;
			\draw [shift={(117.5,605)}, rotate = 0] [fill={rgb, 255:red, 0; green, 0; blue, 0 }  ][line width=0.08]  [draw opacity=0] (8.93,-4.29) -- (0,0) -- (8.93,4.29) -- cycle    ;

			\draw (437,243) node [anchor=north west][inner sep=0.75pt]   [align=left] {\textbf{CSP}};
			
			\draw (84,243) node [anchor=north west][inner sep=0.75pt]   [align=left] {\textbf{Analyst}};
			
			\draw (180.5,321) node [anchor=north west][inner sep=0.75pt]   [align=left] {Run$\displaystyle \ \mathsf{HHE.KeyGen}$};
			
			\draw (182.5,375) node [anchor=north west][inner sep=0.75pt]   [align=left] {$\displaystyle (\mathsf{pk_{A} ,sk_{A} ,evk_{A}}) \ $};
			
			\draw (145,405.4) node [anchor=north west][inner sep=0.75pt]    {$m_{1} =\ \langle t_{1} ,\ \mathsf{Enc( pk_{CSP} ,\ evk_{A})} ,\ \sigma_{A}(\mathsf{H}( t_{1} \ || \ \mathsf{evk_A})) \rangle $};
			
			\draw (830,439) node [anchor=north west][inner sep=0.75pt]   [align=left] {Run$\displaystyle \ \mathsf{HHE.Enc}$};
			
			\draw (830,485) node [anchor=north west][inner sep=0.75pt]   [align=left] {$\displaystyle \mathsf{K_i} ,\ c_{x_i},\ c_\mathsf{K_i}$};
			
			\draw (527.5,555) node [anchor=north west][inner sep=0.75pt]   [align=left] {Run$\displaystyle \ \mathsf{HHE.Decomp}$};
			
			\draw (527.5,643) node [anchor=north west][inner sep=0.75pt]   [align=left] {Run$\displaystyle \ \mathsf{HHE.Eval}$};
			
			\draw (532,693) node [anchor=north west][inner sep=0.75pt]   [align=left] {$\displaystyle c_{res}$};
			
			\draw (535,600) node [anchor=north west][inner sep=0.75pt]   [align=left] {$\displaystyle c'_{x_i}$};
			
			\draw (175,689) node [anchor=north west][inner sep=0.75pt]    {$m_{4} =\ \langle t_{4} ,\ c_{res} ,\ \sigma_{CSP}(\mathsf{H}( t_{4} \ || \ c_{res} ))\rangle $};
			
			\draw (182.5,726) node [anchor=north west][inner sep=0.75pt]   [align=left] {Run$\displaystyle \ \mathsf{HHE.Dec}$};
			
			\draw (184,775) node [anchor=north west][inner sep=0.75pt]   [align=left] {$\displaystyle res$};
			
			\draw (70.19,390.44) node [anchor=north west][inner sep=0.75pt]  [rotate=-270.41] [align=left] {\textbf{Setup}};
			
			\draw (70.19,505.44) node [anchor=north west][inner sep=0.75pt]  [rotate=-270.41] [align=left] {\textbf{Upload}};
			
			\draw (70.19,640.44) node [anchor=north west][inner sep=0.75pt]  [rotate=-270.41] [align=left] {\textbf{Eval}};
			
			\draw (70.19,790.44) node [anchor=north west][inner sep=0.75pt]  [rotate=-270.41] [align=left] {\textbf{Classify}};
			
			\draw (736,243) node [anchor=north west][inner sep=0.75pt]   [align=left] {\textbf{User}};
			
			\draw (498,504) node [anchor=north west][inner sep=0.75pt]    {$m_{2} =\ \langle t_{2} ,\ c_{x_i},\ c_\mathsf{K_i} ,\ \sigma _{u_{i}}(\mathsf{H}( t_{2} \ || \ c_{x_i}\ || \ c_\mathsf{K_i})) \rangle $};
			
			\draw (163,612.4) node [anchor=north west][inner sep=0.75pt]    {$m_{3}=\ \langle t_{3},\ (c_{w}, c_{b}),\ \sigma _A( \mathsf{H}( t_{3} \ || \ (c_w, c_b))) \rangle $};
			
			\draw (182.5,542) node [anchor=north west][inner sep=0.75pt]   [align=left] {Run$\displaystyle \ \mathsf{HE.Enc}$};
			
			\draw (182.5,586) node [anchor=north west][inner sep=0.75pt]   [align=left] {$\displaystyle (c_w, c_b)$};
		\end{tikzpicture}
	}
	\caption{$\mathsf{\threePervPPML}$}
	\label{fig:threeparty}
\end{figure}

\noindent\textbf{\textit{Formal Construction:}} 
The overall flow of the protocol remains the same as in the two-party protocol. \autoref{fig:threeparty} provides a visual overview of the protocol.

\smallskip
\noindent\framebox{$\mathsf{\mathbf{\threePervPPML.Setup}}$:}
Each party generates their respective signing/verification key pair for the signature scheme $\sigma$ and publishes the verification keys. \textbf{A} executes the $\mathsf{HHE.KeyGen}$ algorithm 
to output the public, private and evaluation keys ($\mathsf{pk_{A}, sk_{A}, evk_{A}}$). The steps of the phase are summarized below. 
\begin{myframe}{}
\begin{multicols}{2}
\begin{scriptsize}
    \underline{$u_i$ computes}:
    
    	$\mathsf{(pk_{u_{i}}, sk_{u_{i}})} \leftarrow \mathsf{PKE.Gen(1^{\lambda})}$

    \medskip
    
    \underline{\textbf{CSP} computes}:
    
    	$\mathsf{(pk_{CSP}, sk_{CSP})} \leftarrow \mathsf{PKE.Gen(1^{\lambda})}$
    \columnbreak \\
    \underline{\textbf{A} computes}:
    
    $\mathsf{(ver_{A}, sign_{A})} \leftarrow \mathsf{PKE.Gen(1^{\lambda})}$,\\
    $\mathsf{(pk_A, sk_A, evk_A)} \leftarrow \mathsf{HHE.KeyGen(1^{\lambda})}$ 
\end{scriptsize}
\end{multicols}
\end{myframe}
\noindent \textbf{A} then publishes $\mathsf{pk_{A}}$ and sends $\mathsf{evk_{A}}$ to \textbf{CSP} via $m_{1}$. 
\begin{equation*}
    m_{1} = \langle t_{1}, \mathsf{Enc(pk_{CSP}, evk_{A})}, \sigma_A(\mathsf{H}(t_{1} || \mathsf{evk_{A}}))\rangle
\end{equation*}
\noindent\framebox{$\mathsf{\mathbf{\threePervPPML.Upload}}$:}
With the three-party setting, there can be multiple users, namely $u_{i} \in \mathcal{U}$. Hence, each $u_i$ independently runs $\mathsf{HHE.Enc}$, to generate a unique symmetric key $\mathsf{K_{i}}$ through $\mathsf{SKE.Gen}$. 
Then, $u_i$ encrypts their data by computing $\mathsf{SKE.Enc}$, which takes $\mathsf{K_i}$ and $x_{i}$ as input and outputs $c_{x_i}$. Additionally, each $u_i$ homomorphically encrypts its symmetric key $\mathsf{K_{i}}$ using $\mathsf{HE.Enc}$ with inputs $\mathsf{pk_A}$ and $\mathsf{K_i}$ and output $c_\mathsf{K_{i}}$. The steps of this phase are summarized as:
\begin{gather*}
    \mathsf{K_{i}} \leftarrow \mathsf{SKE.Gen(1^{\lambda})},  \\
        \mathsf{SKE.Enc(K_i},x_i) \rightarrow c_{x_i},  \\
	c_\mathsf{K_{i}} \leftarrow \mathsf{HE.Enc(pk_A,\mathsf{K_i})}  \\
\end{gather*}
Finally, $u_i$ sends the encrypted values ($c_{x_i}$ and $c_\mathsf{K_i}$) through $m_{2}$. 
\begin{equation*}
     m_{2} = \langle t_{2}, c_{x_i}, c_{\mathsf{K_{i}}}, \sigma_{u_{i}}(\mathsf{H}(t_{2} || c_{x_i} || c_{\mathsf{K_{i}}}))\rangle
\end{equation*}
\noindent\framebox{$\mathsf{\mathbf{\threePervPPML.Eval}}$:}
The secure evaluation phase in the three-party protocol is initiated by \textbf{A} 
to gain insight into the data provided by any user $u_{i}$. First, \textbf{A} homomorphically encrypts their ML model parameters $(w, b)$ by running the $\mathsf{HE.Enc}$ algorithm, which takes as input $(\mathsf{pk_A}, (w, b))$ and outputs encrypted parameters $(c_{w}, c_{b})$. \textbf{A} then sends $m_{3}$ to the \textbf{CSP} (\autoref{fig:threeparty}). Upon receiving the encrypted model, the \textbf{CSP} transforms $c_{x_i}$ into homomorphic ciphertext $c'_{x_i}$ by running the $\mathsf{HHE.Decomp}$ algorithm which takes as input $\mathsf{evk_A}, c_{x_i},$ and  $c_\mathsf{K_i}$, and outputs $c'_{x_i}$. Subsequently, the \textbf{CSP} runs the $\mathsf{HHE.Eval}$ algorithm, which takes as input $(\mathsf{evk_{A}}, (c_{w}, c_{b}), c'_{x_i})$ to output an encrypted prediction $c_{res}$. Afterwards, $c_{res}$ is sent through $m_{4}$. The phase is summarized below.

\begin{myframe}{}
\begin{multicols}{2}
\begin{scriptsize}
\textbf{A} computes:\\
	$(c_{w}, c_{b}) \leftarrow \mathsf{HE.Enc(pk_A}, (w, b))$, \\ \\ \\
\columnbreak\\
\textbf{CSP} computes:\\
	$c'_{x_i} \leftarrow \mathsf{HHE.Decomp(evk_A}, c_{x_i}, c_\mathsf{K_i})$, \\
	$c_{res} \leftarrow \mathsf{HHE.Eval(evk_{A}}, (c_{w}, c_{b}), c'_{x_i})$ \\ \\
\end{scriptsize}
\end{multicols}
\end{myframe}

\begin{gather*}
    m_{3} = \langle t_{3}, (c_{w}, c_{b}), \sigma_{A}(\mathsf{H}(t_{3} || (c_{w}, c_{b})))\rangle \\
    m_{4} = \langle t_{4}, c_{res}, \sigma_{CSP}(\mathsf{H}(t_{4} || c_{res}))\rangle
\end{gather*}


\section{Threat Model and Security Analysis}
\label{sec:Threat&SecAnal}
In this section, we define the threat model used to prove the security of \protocolTTT{} by formalizing the capabilities of an adversary $\mathcal{ADV}$. 
The PASTA~\cite{dobraunig2021pasta} scheme we adopt as our underlying cryptographic scheme has been proven to be resilient against differential and linear statistical attacks and their variations. The cipher construction incorporates changing linear layers during encryption, which ensures defence against statistical attacks. The authors provide rigorous proof that this layer instantiation is secure, ensuring full diffusion throughout the entire scheme, even in the best-case scenario for an attacker~\cite{dobraunig2021pasta}. PASTA is also secure against algebraic attacks, such as Linearization and Gröbner Basis Attacks. To this end, we define a threat model focusing on the communication between entities in our protocols and not the underlying scheme itself. We consider a powerful adversary $\mathcal{ADV}$ capable of performing a variety of attacks aiming at breaking the security and privacy of the protocols. In general, $\mathcal{ADV}$ is capable of corrupting any number of users and the \textbf{CSP}. However, for \threePervPPMLTTT{}, we assume that \textbf{A} does not collude with the \textbf{CSP}. Additionally, we assume that each entity can verify the owner of a public key. With this assumption, we eliminate the possibility of basic man-in-the-middle attacks. From these definitions and assumptions, we present the following possible attacks:


\begin{myAttack}[Ciphertext Substitution Attack]
	Let $\mathcal{ADV}$ be a malicious adversary. $\mathcal{ADV}$ successfully launches a Ciphertext Substitution Attack if she manages to replace the generated ciphertexts sent by any entity in an indistinguishable way.
\end{myAttack}

\begin{myAttack}[ML Model Unauthorized Access Attack]
	Let $\mathcal{ADV}$ be a malicious adversary. $\mathcal{ADV}$ successfully launches the ML model unauthorized access attack if she manages to learn information about the underlying ML model utilized by either the \textbf{CSP} or the analyst \textbf{A}.
\end{myAttack}

\subsection{Security Analysis}
\label{subsec:secanalysis}
We now prove the security of our protocols in the presence of the adversary defined in \ref{sec:Threat&SecAnal}. 

\begin{proposition}[Ciphertext Substitution Attack Soundness]
	Let $\sigma$ be an EUF-CMA secure signature scheme and $\mathsf{PKE}$ an INC-CPA public key encryption scheme. Then $\mathcal{ADV}$, cannot successfully launch the Ciphertext Substitution Attack against \protocolTTT{}.
\end{proposition}

\begin{proof}
	To successfully perform the Ciphertext Substitution Attack, $\mathcal{ADV}$ needs to successfully attack the \texttt{\twoPervPPML/\threePervPPML.Upload}, {\texttt{\twoPervPPML.Eval}} or {\texttt{\threePervPPML.Eval}} phase of \protocolTTT{} by substituting the actual ciphertexts with a sequence of $\mathcal{ADV}$ generated ciphertexts. To this end, we categorize the attacks into the following cases:
	
	\begin{enumerate}[leftmargin=1.25cm, label=\bfseries Option \arabic*]
		\item {{\color{magenta} ({\texttt{\twoPervPPML/\threePervPPML.Upload}})}:} When a user $u_i$ outsources data to the \textbf{CSP} via $m_{2} = \langle t_{2}, c_{x_i}, c_{\mathsf{K_i}}, \sigma_{u_i}(\mathsf{H}(t_{2} || c_{x_i} || c_{\mathsf{K_i}}))\rangle$, $\mathcal{ADV}$ needs to replace $c_{x_i}$ with $c'_{x_i}$ and $c_{\mathsf{K_i}}$ with $c'_\mathsf{K_i}$. By successfully performing this attack, $\mathcal{ADV}$ can control the outcome of a query to the \textbf{CSP} to manipulate \textbf{A}. To do this, $\mathcal{ADV}$ must: \smallskip
        \begin{itemize}
            \item Generate a symmetric key $\mathsf{K_{ADV}}$;
            \item Use $\mathsf{K_{ADV}}$ to generate a series of ciphertexts $c'$;
            \item Encrypt $\mathsf{K_{ADV}}$ with $\mathsf{pk_{A}}$ or $\mathsf{pk_{u_i}}$ to get $c'_\mathsf{K_{ADV}}$;
            \item Tamper with $m_2$ in an indistinguishable way.
        \end{itemize}
        \smallskip
        The first three tasks are straightforward to achieve. Additionally, substituting $c_{x_i}$ with $c'_{x_i}$ and $c_{\mathsf{K_i}}$  with $c'_{\mathsf{K_i}}$ in the non signature part of $m_2$ is trivial. However, both $c_{x_i}$ and $c'_{x_i}$ are included in the signature; hence, successfully substituting the terms is equivalent to forging $u_i$'s signature. Given the EUF-CMA security of the signature scheme $\sigma$, this can only happen with negligible probability in the security parameter $\lambda$ of $\sigma$.
		\medskip
		\item {{\color{magenta} ({\texttt{\twoPervPPML.Eval}})}:} When \textbf{CSP} returns the results of a secure classification to $u_i$ via $m_{3} = \langle t_{3}, c_{res}, \sigma_{CSP}(\mathsf{H}(t_{3} || c_{res})) \rangle$, $\mathcal{ADV}$ needs to replace $c_{res}$ with $c'_{res}$ for this instance of the attack to be successful. More specifically, since $c_{res}$ is encrypted with $\mathsf{pk_{u_i}}$, $\mathcal{ADV}$ simply encrypts fictitious results $res'$ with $\mathsf{pk_{u_i}}$ to produce $c'_{res}$. However, $c_{res}$ is included in the signature part of $m_3$; hence tampering with $m_3$ requires forging the signature of \textbf{CSP}. Given the EUF-CMA security of the signature scheme $\sigma$, this can only happen with negligible probability.
		
        \medskip
		\item {{\color{magenta} ({\texttt{\threePervPPML.Eval}})}:} When \textbf{A} initiates secure classification via {$m_{3}=\ \langle t_{3},\ (c_{w}, c_{b}),\ \sigma _{A}( \mathsf{H}( t_{3} \ || \ (c_{w}, c_{b}))) \rangle $}, and \textbf{CSP} responds with the secure evaluation via $m_{4} = \langle t_{4}, c_{res}, \sigma_{CSP}$ $(\mathsf{H}(t_{4} || c_{res} \rangle$, $\mathcal{ADV}$ needs to replace $(c_{w}, c_{b})$ with $(c'_{w}, c'_{b})$ and $c_{res}$ with $c'_{res}$ to successfully perform this attack. The proof for this attack is similar to the one provided for the attack on {\texttt{\twoPervPPML.Eval}} with the only difference being that $\mathcal{ADV}$ attacks both $m_3$ and $m_4$ instead of just $m_3$. Using the same reasoning (i.e., the negligible probability of forging the signature), we conclude that $\mathcal{ADV}$ has a negligible probability of tampering with $m_3$ and $m_4$; hence this attack fails.
	\end{enumerate}
\end{proof}
\begin{proposition}[ML Model Unauthorized Access Attack Soundness]
	Let $f$ be a multi-layered ML model and $\mathsf{HE}$ semantically secure encryption scheme. Then $\mathcal{ADV}$ cannot successfully launch the ML Model Unauthorised Access attack for any of the \protocol{} protocols.
\end{proposition}

\begin{proof}
	To successfully launch the ML Model Unauthorized Access attack, $\mathcal{ADV}$ must collude or corrupt multiple entities in \protocol{} depending on the use case. More specifically, $\mathcal{ADV}$ can either attack \twoPervPPMLTTT{} by colluding with multiple users $(u_j)_{j \in S}$ where $S \subseteq [n]$ or attack \threePervPPMLTTT{} by colluding with a user $u_i$ and the \textbf{CSP}. To this end, we distinguish the attack into the following;
 
	
	\begin{enumerate}[leftmargin=1.15cm,label=\bfseries Option \arabic*]
		\item {{\color{magenta} (Attacking \twoPervPPMLTTT)}:} Let $f$ be a multi-layered ML model owned by \textbf{CSP}, 
		and assume that $\mathcal{ADV}$ colludes with $n' = card(S)$ users. On completion of the secure evaluation phase (\autoref{fig:twoparty}), each user receives $res_{i} = f(x_i)$, where $x_i$ is the input from the user. $\mathcal{ADV}$ successfully launches this attack if with the help of the $n'$ colluding users, $\mathcal{ADV}$ can solve for $f$ given $(res_{j})_{j\in S}$ and $(x_{j})_{j\in S}$. With the assumption that $f$ is a multi-layered ML model, the likelihood of this attack is negligible.
		
		\item {{\color{magenta} (Attacking \threePervPPMLTTT{})}:} Let $f$ be a multi-layered ML model owned by \textbf{A} with $(c_{w}, c_{b})$ as the HE encrypted parameters sent to \textbf{CSP} via {$m_{3}=\ \langle t_{3},\ (c_{w}, c_{b}),\ \sigma _{A}( \mathsf{H}( t_{3} \ || \ (c_{w}, c_{b}))) \rangle $}. $\mathcal{ADV}$ successfully launches this attack if given a corrupt user $u_i$ and a corrupt \textbf{CSP}, $\mathcal{ADV}$ successfully retrieves $f$ in the form of $(w, b)$. $\mathsf{HE}$ has been proven semantically secure; hence, the likelihood of $\mathcal{ADV}$ decrypting the ciphertext $(c_{w}, c_{b})$ is considered negligible.
	\end{enumerate}
\end{proof}

\section{Experiments Section}
\label{sec:performanceanalysis}
In this section, we report the process and results of building a privacy-preserving inference protocol on sensitive medical data (ECG) based on our \threePervPPMLTTT{} construction~\autoref{subsec:MLApp}. Subsequently, we extensively evaluated the computational performance of the \twoPervPPMLTTT{} and \threePervPPMLTTT{} protocols (\autoref{subsec:computation}). For these evaluations, we utilized a dummy dataset where each data input was a vector of four random integers, and the weights and biases were also integer vectors of length four. Finally, to provide concrete evidence of the efficiency of our protocols, we compared the performance of our constructions against a basic BFV HE scheme~\autoref{subsec:commparison}. For these experiments, our primary testbed was a commercial desktop with a 12th Generation Intel i7-12700 CPU with~20 cores and~32GB of RAM running on an Ubuntu 20.04 operating system. Furthermore, in all the evaluations, we utilized the SEAL cryptographic library\footnote{\href{https://github.com/microsoft/SEAL}{https://github.com/microsoft/SEAL}} for basic HE operations and PASTA library\footnote{\href{https://github.com/IAIK/hybrid-HE-framework}{https://github.com/IAIK/hybrid-HE-framework}} to implement the secure symmetric cipher. To ensure statistical significance, each experiment was repeated 50 times, with the average results considered to provide a comprehensive overview of the performance of each algorithm under evaluation. 


\subsection{PPML Application}
\label{subsec:MLApp}
In the first phase of our evaluations, we demonstrated the real-world applicability of our \threePervPPMLTTT{} protocol by applying it to a PPML application with a sensitive heartbeat dataset to classify whether a heartbeat is subjected to heart disease or not. More specifically, we employed the MIT-BIH dataset~\cite{moody2001impact}, which is a dataset of human heartbeats obtained from 47 subjects from 1975 to 1979. 
\begin{figure}[H]
	\centering
	\includegraphics[width=\linewidth]{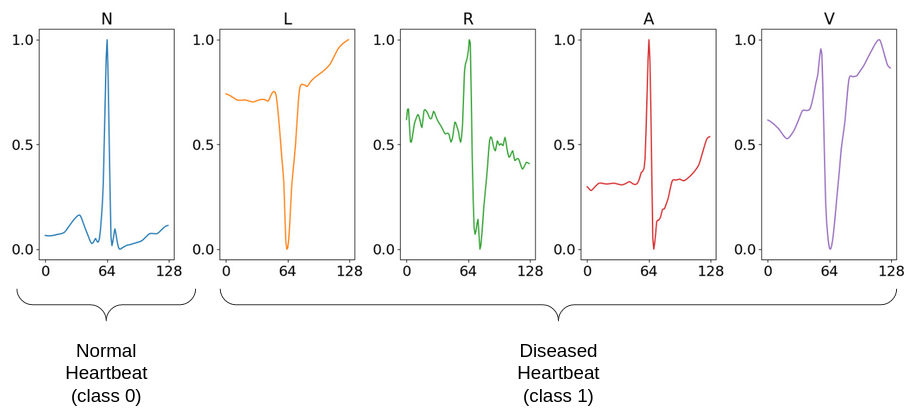}
	\caption{Example heartbeats from the MIT-BIH dataset.  
}
	\label{fig:mitbih_float_signals}
\end{figure}


\subsubsection{MIT-BIH dataset} 
In total, the dataset contains 48 half-hour excerpts of two-channel ECG recordings. We used the processed ECG data from~\cite{abuadbba2020can}, which contains a train split and a test split, each comprising 13,245 ECG examples that belong to five classes: normal heartbeat (N), right bundle branch block (R), left bundle branch block (L), atrial premature contraction (A), ventricular premature contraction (V). Each ECG example is a float 1D time series signal of length 128 with values in the range of $[0, 1]$. We further grouped all ECG examples from the later four classes (L, R, A, V) into one super class called "diseased heartbeats" (\autoref{fig:mitbih_float_signals}). By doing so, we simplified our problem into a binary classification problem. Based on the input ECG signal, we tried to classify if it is a normal or diseased heartbeat. As previously discussed, our HHE protocols are based on the BFV scheme, which only works with integer data and arithmetic. However, our ECG data are floating-point numbers in the range of $[0, 1]$, and normally, training neural networks also produces models with weights and biases in floating-point numbers. 
To this end, we first quantized the ECG data into 4-bit integer data with values in the range of $[0, 15]$ (\autoref{fig:mitbih_float_int_signals}). 

\begin{figure}[H]
	\centering
	\includegraphics[width=.9\linewidth]{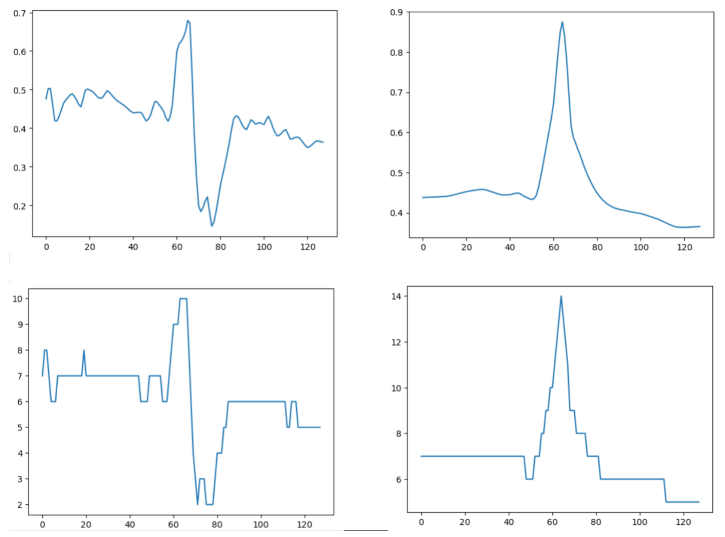}
	\caption{Quantizing ECG data into 4-bit integers. Top: Floating-point ECG data. Bottom: Corresponding quantized ECG data.}
	\label{fig:mitbih_float_int_signals}
\end{figure}
\vspace{-1em}

On preparing the data, we trained a simple neural network with one fully connected (FC) layer with a sigmoid activation function on the floating-point ECG dataset. The neural network can be written as $f(\boldsymbol{\theta}) = \text{sigmoid}(wx + b)$, where $\boldsymbol{\theta} = (w, b)$ are the model's weights and biases, and $x$ is the training data. For the integer ECG data, we used the PocketNN framework~\cite{song2022pocketnn} to train the same neural network in integer arithmetic. In both cases, we trained until the models' predictive accuracy on the test data split no longer improved. Training in floating-point arithmetic produced the best train accuracy of 89.36\% at epoch 482 and the best test accuracy of 88.93\% at epoch 500. With respect to training in integer arithmetic, we got the best train accuracy of 86.65\% at epoch 42 and the best test accuracy of 87.06\% at epoch 31. From these results, we observed that the test accuracy of the integer neural network was only 1.87\% lower than the test accuracy produced by the floating-point neural network. Note that when training in integer arithmetic, we constrained the values of the weight and biases to be in the range of $[-2047, 2048]$, or $[-2^{11}+1, ~2^{11}]$. The reason for this is that the HHE protocol works in $\mathbb{Z}_q$ where $q=2^{16}+1$; therefore, all computation results need to be in the range of $[-2^{15}+1, ~2^{15}]$; otherwise they will be wrapped around by the modulo operation and produce incorrect predictions. As mentioned above, our ECG data are 4-bit integers and in the range of $[0, 2^4 - 1]$. In the 1 FC neural network, we needed to do one encrypted element-wise matrix multiplication between the data and the weights, and hence, if the weights are in $[-2^{11}+1, ~2^{11}]$, the results will be constrained in  $[-2^{11}+1, ~2^{11}] \times 2^{4} = [-2^{15}+2^4, ~2^{15}] \approx [-2^{15}+1, ~2^{15}],$ which is what we needed. After training and getting the trained integer model, we built \ecgPPML{} -- a privacy-preserving inference protocol on the quantized integer ECG data based on our \threePervPPMLTTT{} protocol construction and ran the experiments with the results reported below. %

To begin our evaluations, we ran the encrypted inference protocol on a different number of examples from the test split, then compared the predictions with the ground-truth outputs to get the encrypted test accuracy. For each experiment, we also made inferences on plaintext ECG data in floating point and integer arithmetic to compare the results. \autoref{tab:ecgPPMLAcc} shows the encrypted and plaintext accuracies on a different number of data inputs. Overall, the accuracy of plaintext inference in integer arithmetic is very similar to encrypted inference accuracies. When the number of input examples is low (1-500), integer arithmetic and encrypted accuracies are comparable or even higher than plaintext floating-point inference. When the number of data inputs increases (1000-2000 examples), plaintext integer and encrypted inference produce slightly lower accuracies (0.5-0.8\% lower). For 1000-2000 input data samples, encrypted inference had higher accuracy than plaintext integer inference but with a minimal margin (0.1-0.15\%). We note that this is due to HE noise which helps make a few more correct predictions.

\begin{table}[ht!]
	\centering
	\scalebox{0.9}{
		\begin{tabular}{|c|c|c|c|c|c|c|}
			\hline
			\rowcolor{frenchblue}
			\color{white}{Data Inputs} & \color{white}{Plaintext (Float)} & \color{white}{Plaintext (Integer)} & \color{white}{~~Encrypted~~} \\
			\hline 
			1 &  100 \% & 100 \% & 100 \%\\
			\hline
			\rowcolor{Gray}
			10 & 90 \% & 90 \% & 90 \% \\
			\hline
			20 &  90 \% & 95 \% & 90 \%  \\
			\hline
			\rowcolor{Gray}
			50 &  88 \% & 92 \% & 90 \% \\
			\hline
			100 & 86 \% & 91 \% & 90 \% \\
			\hline
            \rowcolor{Gray}
			500 & 87 \% & 87.2 \% & 86.8 \% \\
			\hline
            1000 & 87.9 \% & 87.3 \% & 87.4 \% \\
			\hline
            \rowcolor{Gray}
            2000 & 88.2 \% & 87.4 \% & 87.55 \% \\
			\hline
	\end{tabular}}
	\caption{Accuracy Analysis -- \ecgPPML{}  \label{tab:ecgPPMLAcc}}
\end{table}

\vspace{-1em}

Subsequently, we evaluated the computational overhead of the \ecgPPML{} protocol (\autoref{tab:ecgPPMLComp}). In the integer and float plaintext inference protocol, the client and analyst are not required to perform any computation as they outsource all their data and neural network model to the CSP. Therefore, in~\autoref{tab:ecgPPMLComp}, the client and analyst only have a single column for the encrypted inference results. Looking at these results (i.e.\ across all data input examples for the encrypted inference protocol), we observed that the CSP is responsible for the most computational overhead (99\% or even more), which increased linearly with the number of data inputs. Compared to plaintext inference, encrypted inference is more computationally expensive. However, encrypted inference for one data sample takes~12.18 seconds on a commercial desktop, which is a promising result. Furthermore, these experimental results align with our goal and expectation for the HHE protocol, as we want the client and analyst to do minimal work, and most of the computations take place in the CSP. Finally, we analyzed the communication cost of the \ecgPPML{} protocol and reported the results in~\autoref{tab:ecgPPMLComm}. For the communication between the client and the CSP, we observed that when the number of data inputs was low (1-500), the encrypted communication cost was very high compared to the plaintext costs due to the size of the HE ciphertext of the symmetric key (1.8 Mb). However, once the number of input examples increased, the size of the symmetrically encrypted data being sent to the CSP increased linearly, similar to the plaintext size, making this difference unimportant. There is no communication between the client and the analyst in plaintext protocols, while in the encrypted inference protocol, the client only transfers the HE public key to the analyst, which has a fixed size of 2.06 Mb. Overall, the communication cost for the client was minimal and increased linearly with the number of data inputs submitted to the CSP. The communication cost between the analyst and CSP was also minimal for plaintext inference since the plaintext weights, biases, and results were small. On the other hand, the majority of the communication cost for the encrypted inference protocol was incurred between the analyst and CSP. This increased cost is caused by the HE-encrypted output of the linear layer that needs to be sent from the CSP to the analyst. Hence, if the number of data inputs increases, the communication cost between the analyst and CSP will increase linearly. We observed that the communication cost for 2000 data input examples is 5548.21 Mb, or about 5 Gb of data, which is a reasonable result for today's internet bandwidth.

\begin{table}[ht!]
\centering
\scalebox{0.7}{\begin{tabular}{|c|c|c|c|c|c|} 
\hline
\rowcolor{frenchblue}
\textcolor{white}{Data Inputs} & \textcolor{white}{Client}    & \textcolor{white}{Analyst}   & \multicolumn{3}{c|}{\textcolor{white}{CSP }}                                                                                                                             \\ 
\hline
\rowcolor[rgb]{0.208,0.518,0.894}                               & \textcolor{white}{Encrypted} & \textcolor{white}{Encrypted} & \multicolumn{1}{c|}{\textcolor{white}{Plaintext (float)}} & \multicolumn{1}{c|}{\textcolor{white}{Plaintext (integer)}} & \textcolor{white}{Encrypted}                 \\ 
\hline
1                                                               & 0.16                         & 0.52                          & 0                                                           & 0.23                                                        & 12.18                                        \\ 
\hline
\rowcolor{Gray} 10                            & 0.18                         & 0.58                         & 0                                                           & 0.24                                                        & 120.44   \\ 
\hline
20                                                              & 0.2                         & 0.63                         & 0                                                           & 0.23                                                        & 241.21                                       \\ 
\hline
\rowcolor{Gray} 50                            & 0.27                         & 0.803                         & 0                                                           & 0.25                                                        & 601.95                                       \\ 
\hline
100                                                             & 0.4                         & 1.091                         & 0                                                           & 0.24                                                        & 1212.38                                      \\ 
\hline
\rowcolor{Gray} 500                           & 1.31                         & 3.38                         & 0.05                                                        & 0.24                                                        & \textcolor[rgb]{0.051,0.051,0.051}{6021.98}  \\ 
\hline
1000                                                            & 2.46                         & 6.2                         & 0.1                                                         & 0.25                                                        & \textcolor[rgb]{0.051,0.051,0.051}{12058.2}  \\ 
\hline
\rowcolor{Gray} 2000                          & 4.8                        & 11.92                        & 0.2                                                         & 0.27                                                        & \textcolor[rgb]{0.051,0.051,0.051}{24153.5}  \\ 
\hline
\multicolumn{1}{l}{}                                            & \multicolumn{1}{l}{}         & \multicolumn{1}{l}{}         & \multicolumn{1}{l}{}                                        & \multicolumn{1}{l}{}                                        & \multicolumn{1}{l}{}                        
\end{tabular}}
\caption{Computation Analysis -- \ecgPPML{}. All numbers 
in seconds.
\label{tab:ecgPPMLComp}}
\end{table}

\vspace{-1em}

These experimental results show that our \ecgPPML{} protocol produces comparable results in accuracy compared to inference on plaintext data. Furthermore, the CSP is responsible for most of the computation costs, and the majority of communication cost also occurs between the CSP and the analyst. These results align with our vision of using HHE for PPML applications and show the potential for HHE when applied in real-world PPML applications. 

\begin{table}[ht!]
\centering
\scalebox{0.6}{
\begin{tabular}{|c|c|c|c|c|c|c|c|}
\hline
\rowcolor{frenchblue} \textcolor{white}{Data Inputs} & \multicolumn{3}{c|}{\textcolor{white}{Client - CSP}}                                                                                                                                                                                                                                                                       & \textcolor{white}{Client - Analyst} & \multicolumn{3}{c|}{\textcolor{white}{Analyst - CSP }}                                                                                 \\ 
\hline
\rowcolor[rgb]{0.208,0.518,0.894}                               & \begin{tabular}[c]{@{}>{\cellcolor[rgb]{0.208,0.518,0.894}}c@{}}\textcolor{white}{Plaintext}\\\textcolor{white}{(float)}\end{tabular} & \begin{tabular}[c]{@{}>{\cellcolor[rgb]{0.208,0.518,0.894}}c@{}}\textcolor{white}{Plaintext}\\\textcolor{white}{(integer)}\end{tabular} & \textcolor{white}{Encrypted}             & \textcolor{white}{Encrypted}        & \begin{tabular}[c]{@{}>{\cellcolor[rgb]{0.208,0.518,0.894}}c@{}}\textcolor{white}{Plaintext}\\\textcolor{white}{(float)}\end{tabular} & \begin{tabular}[c]{@{}>{\cellcolor[rgb]{0.208,0.518,0.894}}c@{}}\textcolor{white}{Plaintext}\\\textcolor{white}{(integer)}\end{tabular} & \textcolor{white}{Encrypted}                 \\ 
\hline
1                                                               & \textcolor[rgb]{0.024,0.024,0.024}{0.0002}                                                                                            & \textcolor[rgb]{0.024,0.024,0.024}{0.0002}                                                                                              & \textcolor[rgb]{0.051,0.051,0.051}{1.8} & 2.06                                & \textcolor[rgb]{0.024,0.024,0.024}{0.0017}                                                                                            & \textcolor[rgb]{0.024,0.024,0.024}{0.000734}                                                                                            & 72.46                                        \\ 
\hline
\rowcolor{Gray} 10                            & \textcolor[rgb]{0.024,0.024,0.024}{0.002}                                                                                             & \textcolor[rgb]{0.024,0.024,0.024}{0.002}                                                                                               & 1.8                                     & 2.06                                & \textcolor[rgb]{0.024,0.024,0.024}{0.0017}                                                                                            & \textcolor[rgb]{0.024,0.024,0.024}{0.000734}                                                                                            & 97.11                                        \\ 
\hline
20                                                              & \textcolor[rgb]{0.024,0.024,0.024}{0.005}                                                                                             & \textcolor[rgb]{0.024,0.024,0.024}{0.005}                                                                                               & 1.81                                     & 2.06                                & \textcolor[rgb]{0.024,0.024,0.024}{0.0017}                                                                                            & \textcolor[rgb]{0.024,0.024,0.024}{0.000734}                                                                                            & 124.51                                       \\ 
\hline
\rowcolor{Gray} 50                            & \textcolor[rgb]{0.024,0.024,0.024}{0.012}                                                                                             & \textcolor[rgb]{0.024,0.024,0.024}{0.012}                                                                                               & 1.81                                     & 2.06                                & \textcolor[rgb]{0.024,0.024,0.024}{0.0017}                                                                                            & \textcolor[rgb]{0.024,0.024,0.024}{0.000734}                                                                                            & 206.692                                      \\ 
\hline
100                                                             & \textcolor[rgb]{0.024,0.024,0.024}{0.029}                                                                                             & \textcolor[rgb]{0.024,0.024,0.024}{0.029}                                                                                               & 1.83                                     & 2.06                                & \textcolor[rgb]{0.024,0.024,0.024}{0.0017}                                                                                            & \textcolor[rgb]{0.024,0.024,0.024}{0.000734}                                                                                            & 343.643                                      \\ 
\hline
\rowcolor{Gray} 500                           & 1.1                                                                                                                                   & \textcolor[rgb]{0.024,0.024,0.024}{1.1}                                                                                                 & \textcolor[rgb]{0.051,0.051,0.051}{2.9} & 2.06                                & \textcolor[rgb]{0.024,0.024,0.024}{0.0017}                                                                                            & \textcolor[rgb]{0.024,0.024,0.024}{0.000734}                                                                                            & \textcolor[rgb]{0.051,0.051,0.051}{1439.27}  \\ 
\hline
1000                                                            & 2.3                                                                                                                                   & \textcolor[rgb]{0.024,0.024,0.024}{2.3}                                                                                                 & \textcolor[rgb]{0.051,0.051,0.051}{4.1} & 2.06                                & \textcolor[rgb]{0.024,0.024,0.024}{0.0017}                                                                                            & \textcolor[rgb]{0.024,0.024,0.024}{0.000734}                                                                                            & \textcolor[rgb]{0.051,0.051,0.051}{2809.02}  \\ 
\hline
\rowcolor{Gray} 2000                          & 4.6                                                                                                                                   & \textcolor[rgb]{0.024,0.024,0.024}{4.6}                                                                                                 & \textcolor[rgb]{0.051,0.051,0.051}{6.4} & 2.06                                & \textcolor[rgb]{0.024,0.024,0.024}{0.0017}                                                                                            & \textcolor[rgb]{0.024,0.024,0.024}{0.000734}                                                                                            & \textcolor[rgb]{0.051,0.051,0.051}{5548.21}  \\ 
\hline
\multicolumn{1}{l}{}                                            & \multicolumn{1}{l}{}                                                                                                                  & \multicolumn{1}{l}{}                                                                                                                    & \multicolumn{1}{l}{}                     & \multicolumn{1}{l}{}                & \multicolumn{1}{l}{}                                                                                                                  & \multicolumn{1}{l}{}                                                                                                                    & \multicolumn{1}{l}{}                        
\end{tabular}
}
\caption{Communication Analysis -- \ecgPPML{}. All numbers are in Megabytes (Mb). \label{tab:ecgPPMLComm}}
\end{table}
\vspace{-1em}

\subsection{Computational Analysis}
\label{subsec:computation}
This subsection focused on the computational performance of the core algorithms executed by each entity in the \twoPervPPMLTTT{} and \threePervPPMLTTT{} protocols. More precisely, we measured the time taken to execute each HHE algorithm in each protocol phase by the responsible party. For both the {\texttt{\twoPervPPML.Setup}} and {\texttt{\threePervPPML.Setup}} phases, we observed the time taken to generate a set of HHE keys at the user and the analyst, respectively. For implementation purposes, the $\mathsf{HHE.KeyGen}$ algorithm involved the generation of encryption parameters $\mathsf{parms}$, secret key $\mathsf{sk}$, public key $\mathsf{pk}$, relinkey $\mathsf{rk}$, and a Galois key $\mathsf{gk}$, and took~243 milliseconds to execute. Subsequently, in the {\texttt{\twoPervPPML.Upload}} and  {\texttt{\threePervPPML.Upload}} phases, we measured the time taken to homomorphically encrypt a symmetric key $\mathsf{K}$ using the $\mathsf{HE.Enc}$ 
and the time taken to execute the $\mathsf{SKE.Enc}$ 
for various 
inputs ranging from~1 to~300 (each 
input is an integer vector of length 4). $\mathsf{HE.Enc}$ took~7 milliseconds to run. For a single data input, $\mathsf{SKE.Enc}$ executed in~2 milliseconds and~600 milliseconds for~300 data inputs (\autoref{tab:execution}). 

\begin{table}[ht!]
	\centering
	\scalebox{0.75}{
		\begin{tabular}{|c|c|c|c|c|c|}
			\hline
			\rowcolor{frenchblue}		
			\color{white}{Inputs} & \color{white}{$\mathsf{SKE.Enc}$} & \color{white}{$\mathsf{HHE.Decomp}$} & \color{white}{$\mathsf{HHE.Eval}$ (2P)} & \color{white}{$\mathsf{HHE.Eval}$ (3P)} & \color{white}\textbf{$\mathsf{HHE.Dec}$} \\
			\hline 
			1 &  2 ms  &  11.9 s &  7 ms & 0.038 s & 3 ms\\
			\hline
			\rowcolor{Gray}
			50 &  100 ms  &  599.1 s & 350 ms & 1.96 s & 150 ms \\
			\hline
			100 &  200 ms  & 1197.7 s &  700 ms & 3.93 s & 300 ms \\
			\hline
			\rowcolor{Gray}
			150 &  300 ms  & 1794.5 s & 1050 ms & 5.77 s & 450 ms \\
			\hline
			200 & 400 ms & 2394.2 s & 1400 ms & 7.69 s & 600 ms\\
			\hline
			\rowcolor{Gray}
			250 & 500 ms & 2989.2 s & 1750 ms & 9.61 s & 750 ms\\
			\hline
			300 & 600 ms & 3595.6 s & 2100 ms & 11.61 s & 900 ms\\
			\hline
	\end{tabular}}
	\caption{Computational Analysis \label{tab:execution}}
\end{table}

When evaluating the {\texttt{\twoPervPPML.Eval}} phase; we measured the cost of executing the $\mathsf{HHE.Decomp}$ algorithm for various numbers of symmetric ciphertexts from~1 to~300, and the cost of executing the $\mathsf{HHE.Eval}$ algorithm for various homomorphic ciphertexts (1 to~300). For a single input, $\mathsf{HHE.Decomp}$ took~11.9 seconds, while $\mathsf{HHE.Eval}$ took~7 milliseconds. On the other hand, for~300 inputs, $\mathsf{HHE.Decomp}$ took~3595.6 seconds, while $\mathsf{HHE.Eval}$ took~2100 milliseconds (\autoref{tab:execution}). When evaluating \texttt{\threePervPPML.Eval}, we first 
measured the performance of $\mathsf{HE.Enc}$ 
at \textbf{A} and then the performance of the $\mathsf{HHE.Decomp}$ and $\mathsf{HHE.Eval}$ algorithms
on various number of symmetric ciphertext inputs from~1 to~300 at the \textbf{CSP}. $\mathsf{HE.Enc}$ took~16 milliseconds to execute, while the results for $\mathsf{HHE.Decomp}$ were similar to those from {\texttt{\twoPervPPML.Eval}}. For a single input, $\mathsf{HHE.Eval}$ took~38 milliseconds to execute and~11.6 seconds for~300 inputs.

\begin{table}[ht!]
	\centering
	\scalebox{1}{
		\begin{tabular}{|c|c|c|c|c|c|c}
			\hline
			\rowcolor{frenchblue}		
			\color{white}{Phase} & \color{white}{User} & \color{white}{Server} & \color{white}{Total}\\
			\hline 
			$\textbf{{\twoPervPPML.Setup}}$ & 243 ms & -- & 243 ms \\
			\hline
			\rowcolor{Gray}
			$\textbf{{\twoPervPPML.Upload}}$ & 607 ms & -- & 607 ms \\
			\hline
			$\textbf{{\twoPervPPML.Eval}}$ & -- & 3597.7 s & 3597.7 s \\
			\hline
			\rowcolor{Gray}
			$\textbf{{\twoPervPPML.Classify}}$ & 900 ms & -- & 900 ms \\
			\hline
	\end{tabular}}
	\caption{Total Computation Cost -- \twoPervPPML~for 300 data inputs\label{tab:2partyT}}
\end{table}

\noindent Finally, in both {\texttt{\twoPervPPML.Classify}} and {\texttt{\threePervPPML.Classify}} phases, we focused primarily on measuring the cost of 
$\mathsf{HHE.Dec}$ 
for a range of homomorphic ciphertext inputs from~1 to~300. For a single input, $\mathsf{HHE.Dec}$ ran in~3 milliseconds, and for~300 inputs, it ran in~900 milliseconds (\autoref{tab:execution}). \autoref{tab:2partyT} and \autoref{tab:3partyT} provide the computational analysis of 
\twoPervPPMLTTT{} and \threePervPPMLTTT~protocols respectively for~300 
inputs.

\begin{table}[ht!]
	\centering
	\scalebox{0.9}{
		\begin{tabular}{|c|c|c|c|c|c|c|}
			\hline
			\rowcolor{frenchblue}		
			\color{white}{Phase} & \color{white}{Analyst} & \color{white}{User} & \color{white}{Server} & \color{white}{Total}\\
			\hline 
			$\textbf{{\threePervPPML.Setup}}$ & 243 ms & -- & -- & 243 ms \\
			\hline
			\rowcolor{Gray}
			$\textbf{{\threePervPPML.Upload}}$ & -- & 607 ms & -- & 607 ms\\
			\hline
			$\textbf{{\threePervPPML.Eval}}$ & 16 ms & -- & 3607.21 s & 3607.23 s\\
			\hline
			\rowcolor{Gray}
			$\textbf{{\threePervPPML.Classify}}$ & 900 ms & -- & -- & 900 ms\\
			\hline
	\end{tabular}}
	\caption{Total Computation Cost--\threePervPPML~for 300 data inputs\label{tab:3partyT}}
\end{table}

\subsection{Comparison with plain BFV}
\label{subsec:commparison}
To provide concrete evidence of the efficiency of our proposed construction, we implemented a plain BFV scheme with a similar architecture to \threePervPPMLTTT{}~and compared the results. More precisely, we measured the performance of a plain BFV scheme, where a user continuously encrypts data input homomorphically before outsourcing them to the \textbf{CSP}. The same encryption parameters were used for all implementations. For these experiments, we only focused on comparing the total computational and communication costs of running the {\texttt{\threePervPPML.Upload}} phase of our protocol, with the cost of continuously using HE encryption in the plain BFV. We varied the number of data inputs from~1 to~300 (each data input is an integer vector of length 4).

\begin{figure}[ht!]
	\begin{minipage}{.7\linewidth}
		\centering
	\begin{tikzpicture}
		\begin{axis}[
			width=.9\textwidth,
			height=.9\textwidth,
			xlabel={Data (\# input vectors)},
			ylabel={Time (ms)},
			xmin=0, xmax=300,
			ymin=0, ymax=2500,
			xtick={0,50, 100, 150, 200, 250, 300},
			ytick={0, 500, 1000, 1500, 2000, 2500},
			legend pos =north west,
			ymajorgrids=true,
			grid style = dashed,
			]
			\addplot[
			color=blue,
			mark=square,
			]
			coordinates {
				(1,9)(50,107)(100,207)(150,307)(200,407)(250,507)(300,607)
			};
			\addplot[
			color=red,
			mark=square,
			]
			coordinates {
				(1,7)(50,373)(100,748)(150,1145)(200,1540)(250,1906)(300,2240)
			};
			\legend{\threePervPPMLTTT{}, BFV}
		\end{axis}
	\end{tikzpicture}
	\caption{Computation Costs}
	\label{fig:compare1}
	\end{minipage}
	\hfill
	\begin{minipage}{.7\linewidth}
		\centering
	\begin{tikzpicture}
		\begin{axis}[
			width=.9\textwidth,
			height=.9\textwidth,
			xlabel={Data (\# input vectors)},
			ylabel={Size (MB)},
			xmin=0, xmax=300,
			ymin=0, ymax=600,
			xtick={0,50, 100, 150, 200, 250, 300},
			ytick={0, 100, 200, 300, 400, 500, 600},
			legend pos =north west,
			ymajorgrids=true,
			grid style = dashed,
			]
			\addplot[
			color=blue,
			mark=square,
			]
			coordinates {
				(1,1.82)(50,1.83)(100,1.83)(150,1.83)(200,1.84)(250,1.84)(300,1.84)
			};
			\addplot[
			color=red,
			mark=square,
			]
			coordinates {
				(1,1.82)(50,91)(100,183)(150,274)(200, 365)(250,457)(300,548)
			};
			\legend{\threePervPPMLTTT{}, BFV}
		\end{axis}
	\end{tikzpicture}
	\caption{Communication Costs}
	\label{fig:compare2}
\end{minipage}
\end{figure}

For a single data input, \texttt{\threePervPPML.Upload} took~9 milliseconds to execute, while the plain BFV scheme took~7 milliseconds to perform a single HE encryption. It is worth noting that the plain BFV scheme is marginally faster for a single data value. However, this is due to the fact that \texttt{\threePervPPML.Upload} involves two operations (a symmetric encryption operation and an HE encryption operation), while the plain BFV scheme involves just one HE encryption operation. However, when the number of data values 
was increased to 300, \texttt{\threePervPPML.Upload} ran in~0.608 seconds, while the plain BFV scheme ran in~2.2 seconds. \autoref{fig:compare1} provides an overview of the computational comparison results obtained from this phase of our experiments. It is worth pointing out the fact that, in most cases, uploading just one single input is unrealistic since most PPML services will require a plethora of data to properly evaluate a problem. Subsequently, we compared the communication expenses by measuring the total size of transferable ciphertext data in bytes from a user $u_i$ to \textbf{CSP}. Overall, \texttt{\threePervPPML.Upload} sent approximately~1.8 MB of ciphertext data for both a single input and~300 inputs. This is primarily because the size of a symmetric ciphertext is almost negligible as compared to that of a homomorphic ciphertext. The plain BFV, on the other hand, sent approximately~1.82 MB of ciphertext data for a single data input and~547.8 MB for~300 data inputs. \autoref{fig:compare2} provides an overview of the comparison of the communication costs for~1 to~300 different inputs. From these results, it is evident that \protocolTTT{} reduces the communication and computational burden of $u_i$ and transfers them to \textbf{CSP}.


\medskip
\noindent \textbf{\textit{Open Science \& Reproducible Research}}
To support open science and reproducible research and provide other researchers with the opportunity to use, test, and hopefully extend our work, the source codes used for the evaluations have been made available online\footnote{\href{https://github.com/iammrgenie/hhe_ppml}{https://github.com/iammrgenie/hhe\_ppml}}\textsuperscript{,}\footnote{\href{https://github.com/khoaguin/PocketHHE}{https://github.com/khoaguin/PocketHHE}}.

\section{Conclusion}
\label{sec:conclusion}
This paper is one of the first attempts to effectively use the novel concept of HHE to address the problem of privacy-preserving machine learning. We have provided a realistic solution that carefully considers the vagaries of PPML. The designed approach is able to carefully balance ML functionality and privacy so as to allow the use of PPML techniques in a wide range of areas, such as pervasive computing, where, in many cases, the underlying infrastructure presents certain inbuilt limitations. By using HHE, we managed to overcome the main difficulties of PPML application in real-life scenarios, where the majority of data is collected and processed by constraint devices. Certain that the future of cryptography goes hand in hand with ML, we believe we have made the first step towards implementing PPML services with strong security guarantees, which operate efficiently in a wide range of architectures.

\begin{acks}
This work was funded by the HARPOCRATES EU research project (No. 101069535) and the Technology Innovation Institute (TII), UAE, for the project ARROWSMITH.
\end{acks}
%
%
%
\bibliographystyle{ACM-Reference-Format}
\bibliography{guardml}


\begin{thebibliography}{31}


\ifx \showCODEN    \undefined \def \showCODEN     #1{\unskip}     \fi
\ifx \showDOI      \undefined \def \showDOI       #1{#1}\fi
\ifx \showISBNx    \undefined \def \showISBNx     #1{\unskip}     \fi
\ifx \showISBNxiii \undefined \def \showISBNxiii  #1{\unskip}     \fi
\ifx \showISSN     \undefined \def \showISSN      #1{\unskip}     \fi
\ifx \showLCCN     \undefined \def \showLCCN      #1{\unskip}     \fi
\ifx \shownote     \undefined \def \shownote      #1{#1}          \fi
\ifx \showarticletitle \undefined \def \showarticletitle #1{#1}   \fi
\ifx \showURL      \undefined \def \showURL       {\relax}        \fi
\providecommand\bibfield[2]{#2}
\providecommand\bibinfo[2]{#2}
\providecommand\natexlab[1]{#1}
\providecommand\showeprint[2][]{arXiv:#2}

\bibitem[Abuadbba et~al\mbox{.}(2020)]%
        {abuadbba2020can}
\bibfield{author}{\bibinfo{person}{Sharif Abuadbba}, \bibinfo{person}{Kyuyeon
  Kim}, \bibinfo{person}{Minki Kim}, \bibinfo{person}{Chandra Thapa},
  \bibinfo{person}{Seyit~A Camtepe}, \bibinfo{person}{Yansong Gao},
  \bibinfo{person}{Hyoungshick Kim}, {and} \bibinfo{person}{Surya Nepal}.}
  \bibinfo{year}{2020}\natexlab{}.
\newblock \showarticletitle{Can we use split learning on 1d cnn models for
  privacy preserving training?}. In \bibinfo{booktitle}{\emph{Proceedings of
  the 15th ACM Asia Conference on Computer and Communications Security}}.
\newblock


\bibitem[Al~Badawi et~al\mbox{.}(2020)]%
        {al2020towards}
\bibfield{author}{\bibinfo{person}{Ahmad Al~Badawi}, \bibinfo{person}{Chao
  Jin}, \bibinfo{person}{Jie Lin}, \bibinfo{person}{Chan~Fook Mun},
  \bibinfo{person}{Sim~Jun Jie}, \bibinfo{person}{Benjamin Hong~Meng Tan},
  \bibinfo{person}{Xiao Nan}, \bibinfo{person}{Khin Mi~Mi Aung}, {and}
  \bibinfo{person}{Vijay~Ramaseshan Chandrasekhar}.}
  \bibinfo{year}{2020}\natexlab{}.
\newblock \showarticletitle{Towards the alexnet moment for homomorphic
  encryption: Hcnn, the first homomorphic cnn on encrypted data with gpus}.
\newblock \bibinfo{journal}{\emph{IEEE Transactions on Emerging Topics in
  Computing}} \bibinfo{volume}{9}, \bibinfo{number}{3} (\bibinfo{year}{2020}),
  \bibinfo{pages}{1330--1343}.
\newblock


\bibitem[Bakas et~al\mbox{.}(2022)]%
        {bakas2022symmetrical}
\bibfield{author}{\bibinfo{person}{Alexandros Bakas}, \bibinfo{person}{Eugene
  Frimpong}, {and} \bibinfo{person}{Antonis Michalas}.}
  \bibinfo{year}{2022}\natexlab{}.
\newblock \showarticletitle{Symmetrical Disguise: Realizing Homomorphic
  Encryption Services from Symmetric Primitives}. In
  \bibinfo{booktitle}{\emph{International Conference on Security and Privacy in
  Communication Systems}}. Springer, \bibinfo{pages}{353--370}.
\newblock


\bibitem[Bourse et~al\mbox{.}(2018)]%
        {bourse2018fast}
\bibfield{author}{\bibinfo{person}{Florian Bourse}, \bibinfo{person}{Michele
  Minelli}, \bibinfo{person}{Matthias Minihold}, {and} \bibinfo{person}{Pascal
  Paillier}.} \bibinfo{year}{2018}\natexlab{}.
\newblock \showarticletitle{Fast homomorphic evaluation of deep discretized
  neural networks}. In \bibinfo{booktitle}{\emph{Annual International
  Cryptology Conference}}. Springer.
\newblock


\bibitem[Brakerski(2012)]%
        {brakerski2012fully}
\bibfield{author}{\bibinfo{person}{Zvika Brakerski}.}
  \bibinfo{year}{2012}\natexlab{}.
\newblock \showarticletitle{Fully homomorphic encryption without modulus
  switching from classical GapSVP}. In \bibinfo{booktitle}{\emph{Annual
  Cryptology Conference}}. Springer, \bibinfo{pages}{868--886}.
\newblock


\bibitem[Canteaut et~al\mbox{.}(2018)]%
        {canteaut2018stream}
\bibfield{author}{\bibinfo{person}{Anne Canteaut}, \bibinfo{person}{Sergiu
  Carpov}, \bibinfo{person}{Caroline Fontaine}, \bibinfo{person}{Tancr{\`e}de
  Lepoint}, \bibinfo{person}{Mar{\'\i}a Naya-Plasencia},
  \bibinfo{person}{Pascal Paillier}, {and} \bibinfo{person}{Renaud Sirdey}.}
  \bibinfo{year}{2018}\natexlab{}.
\newblock \showarticletitle{Stream ciphers: A practical solution for efficient
  homomorphic-ciphertext compression}.
\newblock \bibinfo{journal}{\emph{Journal of Cryptology}} \bibinfo{volume}{31},
  \bibinfo{number}{3} (\bibinfo{year}{2018}), \bibinfo{pages}{885--916}.
\newblock


\bibitem[Cheon et~al\mbox{.}(2017)]%
        {cheon2017homomorphic}
\bibfield{author}{\bibinfo{person}{Jung~Hee Cheon}, \bibinfo{person}{Andrey
  Kim}, \bibinfo{person}{Miran Kim}, {and} \bibinfo{person}{Yongsoo Song}.}
  \bibinfo{year}{2017}\natexlab{}.
\newblock \showarticletitle{Homomorphic encryption for arithmetic of
  approximate numbers}. In \bibinfo{booktitle}{\emph{International Conference
  on the Theory and Application of Cryptology and Information Security}}.
  Springer, \bibinfo{pages}{409--437}.
\newblock


\bibitem[Chillotti et~al\mbox{.}(2016)]%
        {chillotti2016faster}
\bibfield{author}{\bibinfo{person}{Ilaria Chillotti}, \bibinfo{person}{Nicolas
  Gama}, \bibinfo{person}{Mariya Georgieva}, {and} \bibinfo{person}{Malika
  Izabachene}.} \bibinfo{year}{2016}\natexlab{}.
\newblock \showarticletitle{Faster fully homomorphic encryption: Bootstrapping
  in less than 0.1 seconds}. In \bibinfo{booktitle}{\emph{Advances in
  Cryptology--ASIACRYPT 2016: 22nd International Conference on the Theory and
  Application of Cryptology and Information Security, Hanoi, Vietnam, December
  4-8, 2016, Proceedings, Part I 22}}. Springer, \bibinfo{pages}{3--33}.
\newblock


\bibitem[Chillotti et~al\mbox{.}(2020)]%
        {chillotti2020tfhe}
\bibfield{author}{\bibinfo{person}{Ilaria Chillotti}, \bibinfo{person}{Nicolas
  Gama}, \bibinfo{person}{Mariya Georgieva}, {and} \bibinfo{person}{Malika
  Izabach{\`e}ne}.} \bibinfo{year}{2020}\natexlab{}.
\newblock \showarticletitle{TFHE: fast fully homomorphic encryption over the
  torus}.
\newblock \bibinfo{journal}{\emph{Journal of Cryptology}} \bibinfo{volume}{33},
  \bibinfo{number}{1} (\bibinfo{year}{2020}), \bibinfo{pages}{34--91}.
\newblock


\bibitem[Cho et~al\mbox{.}(2021)]%
        {cho2021transciphering}
\bibfield{author}{\bibinfo{person}{Jihoon Cho}, \bibinfo{person}{Jincheol Ha},
  \bibinfo{person}{Seongkwang Kim}, \bibinfo{person}{ByeongHak Lee},
  \bibinfo{person}{Joohee Lee}, \bibinfo{person}{Jooyoung Lee},
  \bibinfo{person}{Dukjae Moon}, {and} \bibinfo{person}{Hyojin Yoon}.}
  \bibinfo{year}{2021}\natexlab{}.
\newblock \showarticletitle{Transciphering framework for approximate
  homomorphic encryption}. In \bibinfo{booktitle}{\emph{International
  Conference on the Theory and Application of Cryptology and Information
  Security}}. Springer, \bibinfo{pages}{640--669}.
\newblock


\bibitem[Cosseron et~al\mbox{.}(2023)]%
        {cosseron2022towards}
\bibfield{author}{\bibinfo{person}{Orel Cosseron}, \bibinfo{person}{Cl\'{e}ment
  Hoffmann}, \bibinfo{person}{Pierrick M\'{e}aux}, {and}
  \bibinfo{person}{Fran\c{c}ois-Xavier Standaert}.}
  \bibinfo{year}{2023}\natexlab{}.
\newblock \showarticletitle{Towards Case-Optimized Hybrid Homomorphic
  Encryption: Featuring the Elisabeth Stream Cipher}. In
  \bibinfo{booktitle}{\emph{Advances in Cryptology – ASIACRYPT 2022: 28th
  International Conference on the Theory and Application of Cryptology and
  Information Security}} (Taipei, Taiwan).
  \bibinfo{publisher}{Springer-Verlag}, \bibinfo{address}{Berlin, Heidelberg},
  \bibinfo{pages}{32–67}.
\newblock
\showISBNx{978-3-031-22968-8}


\bibitem[Dobraunig et~al\mbox{.}(2023)]%
        {dobraunig2021pasta}
\bibfield{author}{\bibinfo{person}{Christoph Dobraunig},
  \bibinfo{person}{Lorenzo Grassi}, \bibinfo{person}{Lukas Helminger},
  \bibinfo{person}{Christian Rechberger}, \bibinfo{person}{Markus Schofnegger},
  {and} \bibinfo{person}{Roman Walch}.} \bibinfo{year}{2023}\natexlab{}.
\newblock \showarticletitle{Pasta: a case for hybrid homomorphic encryption}.
\newblock \bibinfo{journal}{\emph{Transaction on Cryptographic Hardware and
  Embedded Systems 2023 Issue 3}} (\bibinfo{year}{2023}).
\newblock


\bibitem[Fan and Vercauteren(2012)]%
        {fan2012somewhat}
\bibfield{author}{\bibinfo{person}{Junfeng Fan} {and} \bibinfo{person}{Frederik
  Vercauteren}.} \bibinfo{year}{2012}\natexlab{}.
\newblock \showarticletitle{Somewhat practical fully homomorphic encryption}.
\newblock \bibinfo{journal}{\emph{Cryptology ePrint Archive}}
  (\bibinfo{year}{2012}).
\newblock


\bibitem[Gentry(2009)]%
        {gentry2009fully}
\bibfield{author}{\bibinfo{person}{Craig Gentry}.}
  \bibinfo{year}{2009}\natexlab{}.
\newblock \bibinfo{booktitle}{\emph{A fully homomorphic encryption scheme}}.
\newblock \bibinfo{publisher}{Stanford university}.
\newblock


\bibitem[Gentry et~al\mbox{.}(2012)]%
        {gentry2012homomorphic}
\bibfield{author}{\bibinfo{person}{Craig Gentry}, \bibinfo{person}{Shai
  Halevi}, {and} \bibinfo{person}{Nigel~P Smart}.}
  \bibinfo{year}{2012}\natexlab{}.
\newblock \showarticletitle{Homomorphic evaluation of the AES circuit}. In
  \bibinfo{booktitle}{\emph{Annual Cryptology Conference}}. Springer.
\newblock


\bibitem[Ha et~al\mbox{.}(2022)]%
        {ha2022rubato}
\bibfield{author}{\bibinfo{person}{Jincheol Ha}, \bibinfo{person}{Seongkwang
  Kim}, \bibinfo{person}{Byeonghak Lee}, \bibinfo{person}{Jooyoung Lee}, {and}
  \bibinfo{person}{Mincheol Son}.} \bibinfo{year}{2022}\natexlab{}.
\newblock \showarticletitle{Rubato: Noisy Ciphers for Approximate Homomorphic
  Encryption}. In \bibinfo{booktitle}{\emph{Advances in Cryptology--EUROCRYPT
  2022: 41st Annual International Conference on the Theory and Applications of
  Cryptographic Techniques, Trondheim, Norway, May 30--June 3, 2022,
  Proceedings, Part I}}. Springer, \bibinfo{pages}{581--610}.
\newblock


\bibitem[Hesamifard et~al\mbox{.}(2018)]%
        {hesamifard2018privacy}
\bibfield{author}{\bibinfo{person}{Ehsan Hesamifard}, \bibinfo{person}{Hassan
  Takabi}, \bibinfo{person}{Mehdi Ghasemi}, {and} \bibinfo{person}{Rebecca~N
  Wright}.} \bibinfo{year}{2018}\natexlab{}.
\newblock \showarticletitle{Privacy-preserving machine learning as a service.}
\newblock \bibinfo{journal}{\emph{Proc. Priv. Enhancing Technol.}}
  \bibinfo{volume}{2018}, \bibinfo{number}{3} (\bibinfo{year}{2018}),
  \bibinfo{pages}{123--142}.
\newblock


\bibitem[Khan et~al\mbox{.}(2021)]%
        {khan2021blind}
\bibfield{author}{\bibinfo{person}{Tanveer Khan}, \bibinfo{person}{Alexandros
  Bakas}, {and} \bibinfo{person}{Antonis Michalas}.}
  \bibinfo{year}{2021}\natexlab{}.
\newblock \showarticletitle{Blind faith: Privacy-preserving machine learning
  using function approximation}. In \bibinfo{booktitle}{\emph{2021 IEEE
  Symposium on Computers and Communications (ISCC)}}. IEEE,
  \bibinfo{pages}{1--7}.
\newblock


\bibitem[Khan and Michalas(2023)]%
        {khan2023learning}
\bibfield{author}{\bibinfo{person}{Tanveer Khan} {and} \bibinfo{person}{Antonis
  Michalas}.} \bibinfo{year}{2023}\natexlab{}.
\newblock \showarticletitle{Learning in the Dark: Privacy-Preserving Machine
  Learning using Function Approximation}.
\newblock  (\bibinfo{year}{2023}).
\newblock


\bibitem[Khan et~al\mbox{.}(2023a)]%
        {khan2023more}
\bibfield{author}{\bibinfo{person}{Tanveer Khan}, \bibinfo{person}{Khoa
  Nguyen}, {and} \bibinfo{person}{Antonis Michalas}.}
  \bibinfo{year}{2023}\natexlab{a}.
\newblock \showarticletitle{A More Secure Split: Enhancing the Security of
  Privacy-Preserving Split Learning}. In \bibinfo{booktitle}{\emph{Nordic
  Conference on Secure IT Systems}}. Springer, \bibinfo{pages}{307--329}.
\newblock


\bibitem[Khan et~al\mbox{.}(2023b)]%
        {khan2023split}
\bibfield{author}{\bibinfo{person}{Tanveer Khan}, \bibinfo{person}{Khoa
  Nguyen}, {and} \bibinfo{person}{Antonis Michalas}.}
  \bibinfo{year}{2023}\natexlab{b}.
\newblock \showarticletitle{Split Ways: Privacy-Preserving Training of
  Encrypted Data Using Split Learning}. In \bibinfo{booktitle}{\emph{2023
  Workshops of the EDBT/ICDT Joint Conference, EDBT/ICDT-WS 2023, 28 March
  2023}}. CEUR-WS.
\newblock


\bibitem[Khan et~al\mbox{.}(2023c)]%
        {khan2023love}
\bibfield{author}{\bibinfo{person}{Tanveer Khan}, \bibinfo{person}{Khoa
  Nguyen}, \bibinfo{person}{Antonis Michalas}, {and}
  \bibinfo{person}{Alexandros Bakas}.} \bibinfo{year}{2023}\natexlab{c}.
\newblock \showarticletitle{Love or Hate? Share or Split? Privacy-Preserving
  Training Using Split Learning and Homomorphic Encryption}. In
  \bibinfo{booktitle}{\emph{2023 20th Annual International Conference on
  Privacy, Security and Trust (PST)}}. IEEE Computer Society,
  \bibinfo{pages}{1--7}.
\newblock


\bibitem[Lee et~al\mbox{.}(2022)]%
        {lee2022privacy}
\bibfield{author}{\bibinfo{person}{Joon-Woo Lee}, \bibinfo{person}{HyungChul
  Kang}, \bibinfo{person}{Yongwoo Lee}, \bibinfo{person}{Woosuk Choi},
  \bibinfo{person}{Jieun Eom}, \bibinfo{person}{Maxim Deryabin},
  \bibinfo{person}{Eunsang Lee}, \bibinfo{person}{Junghyun Lee},
  \bibinfo{person}{Donghoon Yoo}, \bibinfo{person}{Young-Sik Kim},
  {et~al\mbox{.}}} \bibinfo{year}{2022}\natexlab{}.
\newblock \showarticletitle{Privacy-preserving machine learning with fully
  homomorphic encryption for deep neural network}.
\newblock \bibinfo{journal}{\emph{IEEE Access}}  \bibinfo{volume}{10}
  (\bibinfo{year}{2022}), \bibinfo{pages}{30039--30054}.
\newblock


\bibitem[Lou et~al\mbox{.}(2020)]%
        {lou2020glyph}
\bibfield{author}{\bibinfo{person}{Qian Lou}, \bibinfo{person}{Bo Feng},
  \bibinfo{person}{Geoffrey Charles~Fox}, {and} \bibinfo{person}{Lei Jiang}.}
  \bibinfo{year}{2020}\natexlab{}.
\newblock \showarticletitle{Glyph: Fast and accurately training deep neural
  networks on encrypted data}.
\newblock \bibinfo{journal}{\emph{Advances in Neural Information Processing
  Systems}}  \bibinfo{volume}{33} (\bibinfo{year}{2020}),
  \bibinfo{pages}{9193--9202}.
\newblock


\bibitem[M{\'e}aux et~al\mbox{.}(2019)]%
        {meaux2019improved}
\bibfield{author}{\bibinfo{person}{Pierrick M{\'e}aux}, \bibinfo{person}{Claude
  Carlet}, \bibinfo{person}{Anthony Journault}, {and}
  \bibinfo{person}{Fran{\c{c}}ois-Xavier Standaert}.}
  \bibinfo{year}{2019}\natexlab{}.
\newblock \showarticletitle{Improved filter permutators for efficient FHE:
  Better instances and implementations}. In
  \bibinfo{booktitle}{\emph{International Conference on Cryptology in India}}.
  Springer.
\newblock


\bibitem[Moody and Mark(2001)]%
        {moody2001impact}
\bibfield{author}{\bibinfo{person}{George~B Moody} {and}
  \bibinfo{person}{Roger~G Mark}.} \bibinfo{year}{2001}\natexlab{}.
\newblock \showarticletitle{The impact of the MIT-BIH arrhythmia database}.
\newblock \bibinfo{journal}{\emph{IEEE engineering in medicine and biology
  magazine}} \bibinfo{volume}{20}, \bibinfo{number}{3} (\bibinfo{year}{2001}),
  \bibinfo{pages}{45--50}.
\newblock


\bibitem[Nguyen et~al\mbox{.}(2023)]%
        {nguyen2023split}
\bibfield{author}{\bibinfo{person}{K Nguyen}, \bibinfo{person}{T Khan}, {and}
  \bibinfo{person}{A Michalas}.} \bibinfo{year}{2023}\natexlab{}.
\newblock \showarticletitle{Split Without a Leak: Reducing Privacy Leakage in
  Split Learning}. In \bibinfo{booktitle}{\emph{19th EAI International
  Conference on Security and Privacy in Communication Networks
  (SecureComm’23)}}. Springer.
\newblock


\bibitem[Rivest et~al\mbox{.}(1978)]%
        {rivest1978data}
\bibfield{author}{\bibinfo{person}{Ronald~L Rivest}, \bibinfo{person}{Len
  Adleman}, \bibinfo{person}{Michael~L Dertouzos}, {et~al\mbox{.}}}
  \bibinfo{year}{1978}\natexlab{}.
\newblock \showarticletitle{On data banks and privacy homomorphisms}.
\newblock \bibinfo{journal}{\emph{Foundations of secure computation}}
  \bibinfo{volume}{4}, \bibinfo{number}{11} (\bibinfo{year}{1978}),
  \bibinfo{pages}{169--180}.
\newblock


\bibitem[Sanyal et~al\mbox{.}(2018)]%
        {sanyal2018tapas}
\bibfield{author}{\bibinfo{person}{Amartya Sanyal}, \bibinfo{person}{Matt
  Kusner}, \bibinfo{person}{Adria Gascon}, {and} \bibinfo{person}{Varun
  Kanade}.} \bibinfo{year}{2018}\natexlab{}.
\newblock \showarticletitle{TAPAS: Tricks to accelerate (encrypted) prediction
  as a service}. In \bibinfo{booktitle}{\emph{International Conference on
  Machine Learning}}. PMLR, \bibinfo{pages}{4490--4499}.
\newblock


\bibitem[Sav et~al\mbox{.}(2021)]%
        {sav2020poseidon}
\bibfield{author}{\bibinfo{person}{Sinem Sav}, \bibinfo{person}{Apostolos
  Pyrgelis}, \bibinfo{person}{Juan~R Troncoso-Pastoriza},
  \bibinfo{person}{David Froelicher}, \bibinfo{person}{Jean-Philippe Bossuat},
  \bibinfo{person}{Joao~Sa Sousa}, {and} \bibinfo{person}{Jean-Pierre Hubaux}.}
  \bibinfo{year}{2021}\natexlab{}.
\newblock \showarticletitle{POSEIDON: privacy-preserving federated neural
  network learning}. In \bibinfo{booktitle}{\emph{28th Annual Network and
  Distributed System Security Symposium, {NDSS} 2021, virtually, February
  21-25, 2021}}. \bibinfo{publisher}{The Internet Society}.
\newblock


\bibitem[Song and Lin(2022)]%
        {song2022pocketnn}
\bibfield{author}{\bibinfo{person}{Jaewoo Song} {and} \bibinfo{person}{Fangzhen
  Lin}.} \bibinfo{year}{2022}\natexlab{}.
\newblock \showarticletitle{PocketNN: Integer-only Training and Inference of
  Neural Networks via Direct Feedback Alignment and Pocket Activations in Pure
  C++}.
\newblock \bibinfo{journal}{\emph{Proceedings of tinyML Research Symposium}}
  (\bibinfo{year}{2022}).
\newblock


\end{thebibliography}

\appendix

\end{document}